\documentclass[twoside,11pt]{article}

\usepackage{jmlr2e}

\jmlrheading{24}{2023}{1-37}{09/22; Revised 02/23}{04/23}{22-1043}{Riccardo Grazzi, Massimiliano Pontil and Saverio Salzo}

\ShortHeadings{Bilevel Optimization with a Lower-level contraction}{Grazzi, Pontil and Salzo}
\firstpageno{1}

\usepackage[utf8]{inputenc}

\usepackage[utf8]{inputenc}
\usepackage{natbib}
\usepackage{algorithmic}
\usepackage{algorithm}
\usepackage{amsmath}
\usepackage{amssymb}
\usepackage{amsthm}

\usepackage{appendix}
\usepackage{xcolor}
\usepackage{enumitem}
\usepackage{color}
\usepackage{graphicx}
\usepackage{subfigure}
\usepackage{booktabs} %

\usepackage[colorinlistoftodos,prependcaption,textsize=footnotesize]{todonotes}

\usepackage{cleveref}
\usepackage{hyperref}

\crefname{appsec}{appendix}{appendices}

\newcommand{\BlackBox}{\rule{1.5ex}{1.5ex}}  %
\ifdefined\proof
    \renewenvironment{proof}{\par\noindent{\bf Proof\ }}{\hfill\BlackBox\\[2mm]}
\else
    \newenvironment{proof}{\par\noindent{\bf Proof\ }}{\hfill\BlackBox\\[2mm]}
\fi
 
\newtheorem{theorem}{Theorem}
\newtheorem{lemma}[theorem]{Lemma} 
 
\newtheorem{remark}[theorem]{Remark}
\newtheorem{corollary}[theorem]{Corollary}
\newtheorem{definition}[theorem]{Definition}

\newtheorem{assumption}{Assumption}

\newif\ifstandardproofs
\standardproofsfalse

\newcommand{\R}{\mathbb{R}}
\newcommand{\N}{\mathbb{N}}
\newcommand{\Exp}[2][]{\mathbb{E}#1[{#2}#1]}
\newcommand{\Exptilde}[2][]{\tilde{\mathbb{E}}#1[{#2}#1]}
\newcommand{\PP}{\mathbb{P}}
\newcommand{\EE}{\mathbb{E}}
\newcommand{\VV}{\mathbb{V}}
\newcommand{\Var}[2][]{\mathbb{V}#1[{#2}#1]}
\newcommand{\Vartilde}[2][]{\tilde{\mathbb{V}}#1[{#2}#1]}
\newcommand\given[1][]{\:#1\vert\:}
\newcommand{\indep}{\perp \!\!\! \perp}
\newcommand{\norm}[2][]{#1\lVert{#2}#1\rVert}
\newcommand{\CZ}{\mathcal{Z}}
\newcommand{\vopt}[1]{v({#1},\lambda)}
\newcommand{\voptj}[1]{{\bar v({#1},\lambda)}}
\newcommand{\wstoc}[2][t]{w_{#1}(\lambda_{#2})}
\newcommand{\vstocj}[3][k]{v_{{#1}}({#2},\lambda_{{#3}})}
\newcommand{\csid}{c_3}

\newcommand{\mtwo}{\sigma'_{2}}
\newcommand{\mtwoe}{\sigma_{2,\fo}}
\newcommand{\mone}{\sigma'_{1}}
\newcommand{\monee}{\sigma_{1,\fo}}
\newcommand{\boundw}{B}
\newcommand{\sz}[1]{\eta_{{#1}}}
\newcommand{\sigone}{\sigma_{1}}
\newcommand{\sigtwo}{\sigma_{2}}

\newcommand{\grad}{\nabla}
\newcommand{\jac}{\partial}

\newcommand{\fo}{E} %

\newcommand{\Low}{\mu_{1}}
\newcommand{\Lol}{\mu_{2}}
\newcommand{\Bo}{L_{E}}
\newcommand{\Lola}{\bar{\mu}_{1}}
\newcommand{\Lolb}{\bar{\mu}_{2}}

\newcommand{\boundjac}{L_{\Phi}}
\newcommand{\lipw}{\bar{\nu}_{1}}
\newcommand{\lipl}{\bar{\nu}_{2}}
\newcommand{\g}[1]{\hat \nabla f(\lambda_{#1})}
\newcommand{\yc}{c'}
\newcommand{\yd}{c}

\newcommand{\lipwprime}{L_{w'}}
\newcommand{\ceil}[1]{\lceil{#1}\rceil}

\newcommand{\q}{q}
\newcommand{\rhol}{\nu_{2}}
\newcommand{\rhow}{\nu_{1}}
\newcommand{\LPhi}{L_{\Phi}}

\newcommand{\rf}{\rho}
\newcommand{\hrf}{\sigma}

\newcommand{\s}{s}
\newcommand{\stot}{S}
\newcommand{\J}{J}
\newcommand{\D}{\Lambda}

\newcommand{\proj}{P_{\D}}

\newcommand{\MSE}{\text{MSE}_{\hat \nabla f(\lambda)}}

\newcommand{\MSEC}[1]{\text{MSE}_{#1}}

\DeclareMathOperator*{\argmin}{arg\,min}

 \newcommand{\algosize}{\footnotesize}

\newcommand{\CE}{\text{CE}}
\newcommand{\ball}{\mathcal{B}}

\newcommand{\Ic}{\mathcal{I}}

\newcommand{\SID}{SID}

\begin{document}

\title{Bilevel Optimization with a Lower-level Contraction: Optimal Sample Complexity without Warm-start}

\author{Riccardo Grazzi}
\author{\name Riccardo Grazzi \email riccardo.grazzi@iit.it \\
    \addr Computational Statistics and Machine Learning, \\ Istituto Italiano di Tecnologia, Genoa, Italy and \\ University College of London, UK
       \AND
      \name Massimiliano Pontil \email massimiliano.pontil@iit.it \\
    \addr Computational Statistics and Machine Learning, \\ Istituto Italiano di Tecnologia, Genoa, Italy and \\ University College of London, UK
       \AND
      \name Saverio Salzo \email saverio.salzo@iit.it \\
      \addr Universit\'a la Sapienza di Roma, Italy and \\
       Computational Statistics and Machine Learning, \\ Istituto Italiano di Tecnologia, Genoa, Italy 
}

\editor{Francis Bach}

\maketitle

\begin{abstract}%
We analyse a general class of bilevel problems, in which the upper-level problem consists in the minimization of a smooth objective function and the lower-level problem is to find the fixed point of a smooth contraction map. This type of problems include instances of meta-learning, equilibrium models, hyperparameter optimization and data poisoning adversarial attacks.
Several recent works have proposed algorithms which warm-start the lower-level problem, i.e.~they use the previous lower-level approximate solution as a staring point for the lower-level solver.
This warm-start procedure allows one to improve the sample complexity in both the stochastic and deterministic settings, achieving in some cases the order-wise optimal sample complexity.
However, there are situations,
e.g., meta learning and equilibrium models, in which 
the warm-start procedure
is not well-suited or  ineffective.
In this work
we show that without warm-start, it is still possible to achieve order-wise (near) optimal  sample complexity.
In particular, we propose a simple method which uses (stochastic) fixed point iterations at the lower-level and projected inexact gradient descent at the upper-level, that reaches an $\epsilon$-stationary point using $O(\epsilon^{-2})$ and $\tilde{O}(\epsilon^{-1})$ samples for the stochastic and the deterministic setting, respectively. 
Finally,
compared to methods using warm-start, our approach 
yields a simpler analysis that does not need to study the coupled interactions between the upper-level and lower-level iterates.

\end{abstract}
\begin{keywords}
bilevel optimization; warm-start; non-convex optimization; implicit differentiation; hypergradient; sample complexity. 
\end{keywords}

\section{Introduction}
This paper studies bilevel optimization in the context of machine learning and the design of efficient and principled optimization schemes. More specifically, we consider the following general problem
\begin{equation}
\begin{aligned}
\label{mainprobstoch}
&\min_{\lambda \in \Lambda} f(\lambda) := \Exp{\hat\fo(w(\lambda), \lambda, \xi)}\\
&\text{\ subject~to ~}w(\lambda) = \Exp{\hat\Phi(w(\lambda), \lambda, \zeta)},
\end{aligned}
\end{equation}
where $\Lambda \subseteq \R^m$ is closed and convex, 
$\hat \fo\colon \R^d \times \Lambda \times \Xi  \to \R$ and $\hat\Phi\colon \R^d \times \Lambda \times Z  \to \R^d$, $\xi$ and $\zeta$ are two independent random variables with values in $\Xi$ and $Z$, respectively.
In the following we refer to the problem of finding the fixed point $w(\lambda)$ of \eqref{mainprobstoch} as  the \emph{lower-level} (LL) problem, whereas we call the \emph{upper-level} (UL) problem, that of  minimizing $f$.

Many machine learning problems can be naturally cast in the form 
\eqref{mainprobstoch}. %
Important examples are instances of 
hyperparameter optimization~\citep{maclaurin2015gradient,franceschi2017forward,liu2018darts,lorraine2019optimizing,elsken2019neural}, meta-learning~\citep{andrychowicz2016learning,finn2017model,franceschi2018bilevel},  equilibrium models \citep{bai2019deep}, data poisoning attacks \citep{mei2015using,munoz2017towards}, and graph and recurrent neural networks~\citep{almeida1987learning,pineda1987generalization,scarselli2008graph}. 
In the following we define
\begin{equation*}
   E(w,\lambda) : = \Exp{\hat\fo(w, \lambda, \xi)}, \quad \Phi(w, \lambda) : = \Exp{\hat\Phi(w, \lambda, \zeta)}, 
\end{equation*}
and we assume that $\Phi(\cdot, \lambda)$ is a contraction, 
i.e.~Lipschitz continuous with Lipschitz constant less than one. 
An important special case of the LL problem~in~\eqref{mainprobstoch}, which is the one usually considered in the related literature,
is when
\begin{equation}
\label{minminstochprob}
w(\lambda) = \argmin_{w\in \R^d} \Exp{\hat{\mathcal{L}}(w, \lambda, \zeta)}.
\end{equation}
In this case, provided that the objective 
$\mathcal{L}(w, \lambda) := \Exp{\hat{\mathcal{L}}(w, \lambda, \zeta)}$ 
is strongly convex and Lipschitz smooth,
there always exists a sufficiently small $\eta > 0$ such that the gradient descent map 
\begin{equation}\label{eq:gradmap}
\Phi(w,\lambda) : = w - \eta \nabla_1 \mathcal{L}(w, \lambda),
\end{equation}
is a contraction with respect to $w$.

In dealing with Problem~\eqref{mainprobstoch},
we analyse gradient-based methods which exploit approximations of the hypergradient, i.e.~the gradient of $f$ in \eqref{mainprobstoch}.
As shown in \cite{grazzi2020iteration}, the contraction assumption guarantees that $\Phi(\cdot, \lambda)$ has a unique fixed point $w(\lambda)$ and the hypergradient, thanks to the implicit function theorem \citep[Theorem 5.9]{lang2012fundamentals}, always exists and is given by
\begin{equation}
\label{eq:hypergrad}
    \grad f (\lambda) = \grad_2 \fo(w(\lambda), \lambda)  + \jac_2 \Phi (w(\lambda), \lambda)^\top \vopt{w(\lambda)},
\end{equation}
where $\nabla_i E$ and $\partial \Phi_i$
are the gradient and the Jacobian matrix with respect to the $i$-th component of
$E$ and $\Phi$ respectively,  
and $\vopt{w}$ is the solution of the linear system 
\begin{equation}
\label{eq:LS}\tag{LS}
(I - \jac_1 \Phi(w,\lambda)^\top) v = \nabla_1\fo(w,\lambda),
\end{equation}
which is given by $\vopt{w} : = \big(I - \jac_1 \Phi (w, \lambda)^\top \big)^{-1} \grad_1 \fo(w,\lambda)$.

Computing the hypergradient exactly can be impossible or very expensive since it requires to compute the LL and LS solutions $w(\lambda)$ and $\vopt{w(\lambda)}$. 
This is especially true in large-scale machine learning applications where the number of UL and LL parameters $m$ and $d$ can be very large. 
Furthermore, in cases such as
hyperparameter optimization, where $\fo$ is the average loss over the validation set while $\Phi$ is defined in \eqref{eq:gradmap} with $\mathcal{L}$ being the loss over the training set, if the data set is large, $\fo$, $\Phi$ and their derivatives can become very expensive to compute. For this reason, 
relying on stochastic estimators ($\hat \fo$ and $\hat \Phi$) using only a mini-batch of examples becomes crucial for devising scalable methods. 

To address these issues, \textit{approximate implicit differentiation} (AID) methods \citep{pedregosa2016hyperparameter, rajeswaran2019meta, lorraine2019optimizing}, compute the hypergradient by using approximate solutions for the LL and LS problems. \textit{Iterative differentiation methods} (ITD) \citep{maclaurin2015gradient, franceschi2017forward, franceschi2018bilevel, finn2017model} instead directly differentiate the lower-level solver. 
The convergence of those methods to the true hypergradient has been studied in  \citep{grazzi2020iteration} for AID and ITD methods in the deterministic case and in \citep{grazzi2021convergence} for stochastic AID methods. 

By contrast, here we study the convergence rate of a full bilevel procedure
to solve  Problem~\eqref{mainprobstoch},
based on an extension of the AID method presented in \citep{grazzi2021convergence}. Such type of study was started by \cite{ghadimi2018approximation} and was later followed by several works which we discuss in \Cref{se:related}. Concerning ITD-based methods, we note that similar results were proved only in the deterministic setting \citep{ji2021bilevel,ji2022will}.

\textit{Warm-start.} A common procedure to improve the overall performance of bilevel algorithms is that of using as a starting point for the LL (or LS) solver at the current UL iteration, the  LL (or LS) approximate solution found at the previous UL iteration \citep{hong2020two,guo2021randomized,feihu2021biadam,chen2021tighter}. This  strategy, 
which is called \emph{warm-start}, reduces the number of LL (or LS) iterations needed by the bilevel procedure and is thought to be fundamental to achieve the optimal sample complexity \citep{arbel2021amortized}. 
Moreover, warm-start is sometimes accompanied by the use of \textit{large mini-batches} \citep{ji2021bilevel,arbel2021amortized}, i.e.~averages of many samples, to estimate gradients or Jacobians. 
Large mini-batches allow to reduce the number of UL iteration but increase the cost per iteration and ultimately achieve the same sample complexity up to log terms.

\newcommand{\Dtr}{D_i^{\text{tr}}}
\newcommand{\Dval}{D_i^{\text{val}}}

In spite of the above advantages, warm-start presents a major downside: 
\textit{it is not suitable in applications where it is expensive to store the whole LL solution, such as meta-learning}. 
Indeed, meta-learning  consists in leveraging
``common  properties'' between a set of learning tasks  in order to facilitate the learning process.
 We consider a \textit{meta-training} set of $T$ tasks. Each task $i\in \{1,\dots,T\}$ relies on a training and a validation set which we denote by $\Dtr$ and $\Dval$, respectively. The meta-learning optimization problem is a bilevel problem where the UL objective has the form $f(\lambda) = \sum_{i=1}^{T} f_i(\lambda)$ with $f_i(\lambda) := \mathcal{L}(w^i(\lambda), \lambda; \Dval)$ and the LL solution can be written as
\begin{equation}\label{eq:LLmeta}
    w(\lambda) = \argmin_{w \in \R^{T \times d}} \sum_{i=1}^{T} \mathcal{L}(w^i, \lambda; \Dtr),
\end{equation}
where $\mathcal{L}$, $\lambda$ and $w^i$ (the $i$-th row of $w$) are the loss function, the meta-parameters, and task-specific parameters of the $i$-th task, respectively. 
For example, in \citep{franceschi2018bilevel} $w^i$ and $\lambda$  are the parameters of the last linear layer and the representation part of a neural network, respectively.
Note that the minimization in \eqref{eq:LLmeta} can be performed separately for each task. 
Therefore, when $T$ is large, a common strategy is that of solving, at each UL iteration only a small random subset of tasks.

In this context using warm-start  is problematic.
Indeed, if task $j$ is sampled at iteration $s$, applying warm-start consistently would require using, as a starting point for the LL optimization, the solution for that same task $j$ at iteration $s-1$. However, the task $j$ might not be among the sampled tasks at iteration $s-1$.
A possible remedy would be to warm-start by using the last available approximate solution of the LL problem for task $j$. However, this solution might have been computed too many iterations before the current one, ultimately making the warm-start procedure ineffective (see experiments in \Cref{se:metalearningexp}). In addition, the above strategy would need to keep the approximate solutions for all the previous tasks in memory and eventually for all the $T$ tasks, which might be too costly when $T$ and $d$ are large. Indeed, in \Cref{se:metalearningexp} we consider a problem in which the variable $w$ occupies $122$ GB of memory.  Finally, from the theoretical point of view, this requires a novel analysis to handle the related delays.
This discussion suggests that the warm-start strategy currently considered in literature is not well suited for meta-learning, and indeed is seldom used in meta-learning experiments.

We note that similar issues arise also for equilibrium models when dealing with large data sets. Indeed, in the bilevel formulation of equilibrium models (see e.g.\@ \cite{grazzi2020iteration}) the LL problem consists in finding a fixed point representation for each training example and ultimately yields a separable structure as in meta-learning.

\textit{Contributions.} 
In this work we show for the first time that a bilevel procedure that does not rely on warm-start can achieve optimal sample complexity, improving that by \citet{ghadimi2018approximation}. 
Specifically, we make the following contributions.

\begin{itemize}
    \item{\em We extend the SID estimator proposed in \citep{grazzi2021convergence} by using large mini-batches to estimate  $\nabla E$ and $\jac_2 \Phi$. }
    We prove that this improved \SID{} (\Cref{algo1}) has a $O(1/t)$ convergence rate on the \textit{mean squared error} (MSE), where $t$ is the number of iterations of the LL and LS solvers and the mini-batch size. 
    
        \item {\em We analyse the sample complexity of the bilevel procedure in \Cref{algo2} (BSGM) which combines projected inexact gradient descent with the hypergradient estimator computed via \SID{}}. In particular, we prove, without any convexity assumptions on $f$, that BSGM achieves the optimal and near-optimal sample complexities of $O(\epsilon^{-2})$ (with a finite horizon) and $\tilde{O}(\epsilon^{-2})$, to reach an $\epsilon$-stationary point of Problem~\eqref{mainprobstoch}. In addition, it obtains near-optimal complexity of $\tilde{O}(\epsilon^{-1})$ for the deterministic case.
        We stress that these results are achieved without warm-start, although with a reasonable additional assumption (see \Cref{rem:assumptions}\ref{rem:wsass} and \Cref{rm:advws}).

    \item {\em We provide a simple and modular theoretical analysis which also extends previous ones by  considering the more general case where the LL problem is a fixed-point equation instead of a minimization problem and by relaxing some of the assumptions}. In particular, we cover the case where $\lambda$ is subject to constraints (i.e.~when $\Lambda \neq \R^m$), which are often needed to satisfy the other assumptions of the analysis, but neglected by some previous works. We also extend the scope of applicability of the method  by including e.g.~non-Lipschitz LL losses, like the square loss, in problems of type~\eqref{minminstochprob}.
    \item \textit{We evaluate the empirical performance of our method against other methods using warm-start on three instances of the bi-level problem \eqref{mainprobstoch}.} Specifically, we provide experiments on equilibrium models and meta-learning showing that warm-start is ineffective and increases the memory cost. We also perform a data poisoning experiment which shows that warm-start can be beneficial, although our method remains competitive. We provide the code at \url{https://github.com/CSML-IIT-UCL/bioptexps}
\end{itemize}
\vspace{-.5truecm}
\textit{Notation.}
We denote by $\norm{\cdot}$ either the Euclidean norm or the spectral norm (when applied to matrices). 
The transpose
and the inverse of a given matrix $A$, is denoted by $A^\top$ and $A^{-1}$, respectively.
For a real-valued function $g\colon \R^n\times\R^m\to \R$, we denote by $\nabla_1 g(x,y) \in \R^n$ and $\nabla_2 g(x,y) \in \R^m$, the partial derivatives w.r.t.~the first and second variable, respectively.  For a vector-valued function $h\colon \R^n\times \R^m \to \R^k$ we denote by
$\partial_1 h(x,y) \in \R^{k\times n}$
and $\partial_2 h(x,y) \in \R^{k\times m}$ the partial Jacobians w.r.t.~the first and second variables respectively.
For a (matrix or vector) random variable $X$  we denote by $\EE[X]$ and $\VV[X] := \Exp{\norm{X -\Exp{X}}^2}$ its expectation and variance respectively.
Finally, given two random variables $X$ and $Y$,
the conditional variance of $X$ given $Y$ is $\Var{X \given Y} := \Exp{\norm{X -\Exp{X \given Y}}^2 \given Y}$.
We use the shorthand $\jac \Phi^\top v$ to denote $\jac \Phi(w,\lambda)^\top v$ for some $w,\lambda$.

\textit{Organization.} 
In \Cref{se:BSGM} we describe the bilevel procedure.
We discuss closely related works in \Cref{se:related}. 
In \Cref{se:preliminaries} we state our assumptions and some properties of the bilevel problem. In \Cref{se:sid} we analyse the convergence of SID. In \Cref{se:bilevelconv} we first study the convergence of the projected inexact gradient method with controllable mean square error on the gradient, and then combine this analysis with the one in \Cref{se:sid} to derive the desired complexity results for BSGM. We present the experiments in \Cref{se:experiments}.

\section{Bilevel Stochastic Gradient Method}\label{se:BSGM}
We study the simple double-loop procedure in Algorithm~\ref{algo2} (BSGM). BSGM uses projected inexact gradient updates for the UL problem, where the (biased) hypergradient estimator is provided by Algorithm~\ref{algo1} (\SID{}). 
\SID{} computes the hypergradient by first solving the LL problem (Step 1), then it computes the estimator of the partial gradients of the UL function $E$ using mini-batches of size $\J$ (Step 2). After this it computes an approximate solution to the LS (Step 3). Finally, it combines the LL and LS solutions together with min-batch estimators of $\nabla_2 E$ and $\jac_2 \Phi$, both computed using a mini-batch of size $J$, to give the final hypergradient estimator (Step 4). 
We remark that the samplings performed at all the four steps have to be mutually independent. 
Moreover, to solve the LL and LS problems we use simple stochastic fixed-point iterations which reduce to stochastic gradient descent in LL problems of type \eqref{minminstochprob}. 
We use the same sequence of step sizes $\eta_i$ for both the LL and LS solvers and the same batch size $J$ for both $\nabla E$ and $\jac_2 \Phi$ to simplify the analysis and to reduce the number of configuration parameters of the method. While this choice still achieves the optimal sample complexity dependency on $\epsilon$, it may be suboptimal in practice and does not achieve the optimal dependency on the contraction constant (see \Cref{rem:cc}).

\SID{} is an extension of Algorithm~1 in \cite{grazzi2021convergence} which additionally takes mini-batches of size $J$ to reduce the variance in the estimation of $\nabla E$ and $\jac_2 \Phi$. 
Note that while we specify the LL and LS solvers, the analysis of Algorithm 2 in \Cref{se:sid} works for any converging solver, similarly to \cite{grazzi2021convergence}. In particular, one could use variance reduction or acceleration methods to further improve convergence whenever possible.

\begin{figure}[ht]
    \centering

\begin{algorithm}[H]
\caption{Stochastic Implicit Differentiation (\SID{})}
\textbf{Requires:} $t, k, \J, \lambda, w_0, (\eta_i)_{i=0}^\infty$.
\label{algo1}
\begingroup
\renewcommand{\wstoc}[1][t]{w_{#1}(\lambda)}
\renewcommand{\vstocj}[2][k]{ v_{{#1}}({#2},\lambda)}
\begin{enumerate}
\item \textbf{LL Solver:} 
\begin{equation}\label{eq:llsolver}
\begin{array}{l}
\text{for}\;i=0,1,\ldots t-1\\[0.4ex]
    \left\lfloor
    \begin{array}{l}
    \wstoc[i+1] = \wstoc[i] + \sz{i} (\hat \Phi(\wstoc[i], \lambda, \zeta_{i}) - \wstoc[i])
    \end{array}
    \right.
\end{array}
\end{equation}
where $(\zeta_i)_{0 \leq i \leq t -1}$ are i.i.d.~copies of $\zeta$.

\item Compute $ \nabla_{i} \bar \fo_{\J}(\wstoc{}, \lambda) = \frac{1}{\J}\sum_{j=1}^{{\J}} \nabla_i \hat\fo(\wstoc{}, \lambda, \xi_j) $,
where $(\xi_j)_{1 \leq j \leq \J}$
are i.i.d.~copies of $\xi$ and $i \in \{1,2\}$.
\item \textbf{LS Solver:} %
\begin{equation}\label{eq:lssolver}
\begin{array}{l}
\text{for}\;i=0,1,\ldots k-1\\[0.4ex]
    \left\lfloor
    \begin{array}{l}
    \vstocj[i+1]{\wstoc[t]} = \vstocj[i]{\wstoc[t]}  + \sz{i} (\hat \Psi_{\wstoc[t]}(\vstocj[i]{\wstoc[t]}, \lambda, \hat \zeta_{i}) - \vstocj[i]{\wstoc[t]})
    \end{array}
    \right.
\end{array}
\end{equation}
where $\hat \Psi_w(v, \lambda, z) : = \jac_1 \hat \Phi (w, \lambda, z)^\top v + \grad_1 \bar \fo_\J (w, \lambda)$, $(\hat{\zeta}_i)_{0 \leq i \leq k-1}$ are i.i.d.~copies of $\zeta$.
\item Compute the approximate hypergradient as
\begin{align*}
  \hat \grad f(\lambda) : =& \grad_2 \bar \fo_{\J}(\wstoc{}, \lambda) 
    + \jac_2  \bar\Phi_{\J}(\wstoc{},\lambda)^\top \vstocj{\wstoc{}}{}.
      \vspace{-.25truecm}
\end{align*}
where $ \jac_2 \bar\Phi_{\J}(\wstoc{},\lambda) = \frac{1}{\J}\sum_{j=1}^{{\J}} \jac_2 \hat\Phi(\wstoc{}, \lambda, \zeta'_j) $
and $(\zeta'_j)_{1 \leq j \leq \J}$
are i.i.d.~copies of $\zeta$.

\end{enumerate}
\endgroup
\end{algorithm}

\begin{algorithm}[H]
\caption{Bilevel Stochastic Gradient Method (BSGM)}\label{algo2}
\textbf{Requires:} $\lambda_0, w_0, \alpha, \{\eta_j\}, \{t_\s\}, \{\J_\s\}$.

\textbf{for} $\s = 0,1, \dots, S-1$

\begin{enumerate}

\item 
Compute $\hat \nabla f(\lambda_\s)$ using  Algorithm~\ref{algo1} (\SID{}) with $t=t_\s, k=t_\s, \J=\J_\s, \lambda=\lambda_\s$, $\{\eta_i\} = \{\eta_j\}$,  and $w_0=w_0$, $v_0=0$ (no warm-start).

\item 
 $\lambda_{\s+1} = \proj(\lambda_\s - \alpha\hat \nabla f(\lambda_\s))$  
\end{enumerate}

\end{algorithm}

\end{figure}

\section{Comparison with Related Work}\label{se:related}

\begin{table*}[ht]
  \centering

  \resizebox{\textwidth}{!}{
  \begin{tabular}{c|c|c|c|c|c|c|c}
 \toprule
 \textbf{Algorithm}                    &  \textbf{SC}          & \textbf{BS-LL}             & \textbf{WS} & $t_\s$            & $k_\s$              & $\alpha_\s$         & $\sz{t,\s}$             \\ \hline
  {\algosize BSA \citep{ghadimi2018approximation}} & $O(\epsilon^{-3})$            & $\Theta(1)$             & N, N            & $\Theta(\sqrt{\s})$ & $ \Theta(\log(\sqrt{\s}))$ & $\Theta(1/\sqrt{\stot})$ & $\Theta(1/t)$             \\ \hline
  {\algosize TTSA \citep{hong2020two}}             & $\tilde{O}(\epsilon^{-2.5})$  & $\Theta(1)$             & Y, N            & $1$                 & $ \Theta(\log(\sqrt{\s}))$ & $\Theta(\stot^{-2/5})$   & $\Theta(\stot^{-3/5})$    \\ \hline
  {\algosize stocBiO \citep{ji2021bilevel}}        & $\tilde{O}(\epsilon^{-2})$    & $\Theta(\stot)$         & Y, N            & $\Theta(1)$         & $ \Theta(\log(\sqrt{\s}))$ & $ \leq 1/4L_f$      & $\Theta(1)$               \\ \hline
  SMB \citep{guo2021stochastic}        & $\tilde{O}(\epsilon^{-2})$    & $\Theta(1)$             & Y, N            & $1$                 & $ \Theta(\log(\sqrt{\s}))$ & $\Theta(1/\sqrt{\stot})$ & $\Theta(1/\sqrt{\stot})$  \\
  \hline
  {\algosize saBiAdam \citep{feihu2021biadam}}     & $\tilde{O}(\epsilon^{-2})$    & $\Theta(1)$             & Y, N            & $1$         & $\Theta(\log(\sqrt{\s}))$   & $\Theta(1/\sqrt{\s})$     & $\Theta(1/\sqrt{\s})$      \\
  \hline
  {\algosize ALSET \citep{chen2021tighter}}     & $\tilde{O}(\epsilon^{-2})$    & $\Theta(1)$             & Y, N            & $1$         & $\Theta(\log(\sqrt{\stot}))$   & $\Theta(1/\sqrt{\stot})$     & $\Theta(1/\sqrt{\stot})$      \\
  \hline
  {\algosize Amigo \citep{arbel2021amortized}}        & $O(\epsilon^{-2})$    & $\Theta(\stot)$         & Y, Y            & $\Theta(1)$         & $ \Theta(1)$ & $ \leq 1/L_f$      & $\Theta(1)$               \\
  \hline
  {\algosize \textbf{BSGM \Cref{thm:finalbound}\ref{resone}}}              & $\tilde{O}(\epsilon^{-2})$    & $\Theta(1)$                     & N, N            & $\Theta(\s)$        & $ \Theta(\s)$              & $ \leq 1/L_f$       & $\Theta(1/t)$             \\ \hline
  {\algosize \textbf{BSGM \Cref{thm:finalbound}\ref{restwo}}}              & $O(\epsilon^{-2})$            & $\Theta(1)$                     & N, N            & $\Theta(\stot)$     & $ \Theta(\stot)$           & $ \leq 1/L_f$       & $\Theta(1/t)$             \\
  \bottomrule
  {\algosize STABLE \citep{chen2021single}}        & $O(\epsilon^{-2})$            & $\Theta(1)$                     & Y, N            & $1$                 & ESI & $\Theta(1/\sqrt{\stot})$ & $\Theta(1/\sqrt{\stot})$  \\
  \hline
  {\algosize FSLA \citep{li2021fully}}        & $O(\epsilon^{-2})$    & $\Theta(1)$         & Y, Y            & $1$         & $1$ & $ \Theta(1/\sqrt{s}) $      & $\Theta(1/\sqrt{s})$               \\
  \hline
  {\algosize STABLE-VR \citep{guo2021randomized}}  & $\tilde{O}(\epsilon^{-1.5})$  & $\Theta(1)$             & Y, N            & $1$                 & ESI & $\Theta(\s^{-1/3})$      & $\Theta(\s^{-1/3})$       \\ \hline
  {\algosize SUSTAIN \citep{khanduri2021near}}     & $\tilde{O}(\epsilon^{-1.5})$  & $\Theta(1)$             & Y, N            & $1$                 & $ \Theta(\log(\sqrt{\s}))$ & $\Theta(\s^{-1/3})$      & $\Theta(\s^{-1/3})$       \\  \hline
  \algosize VR-saBiAdam \citep{feihu2021biadam}  & $\tilde{O}(\epsilon^{-1.5})$  & $\Theta(1)$             & Y, N            & $1$         & $\Theta(\log(\sqrt{\s}))$   & $\Theta(\s^{-1/3})$       & $\Theta(\s^{-1/3})$       
  \\  \hline
  
  {\algosize MRBO \citep{yang2021provably}}  & $\tilde{O}(\epsilon^{-1.5})$  & $\Theta(1)$             & Y, N            & $1$         & $\Theta(\log(S))$   & $\Theta(\s^{-1/3})$       & $\Theta(\s^{-1/3})$      
  \\  \hline
  {\algosize VRBO \citep{yang2021provably}}  & $\tilde{O}(\epsilon^{-1.5})$  & $\Theta(\sqrt{S})$             & Y, N            & $\Theta(1)$         & $\Theta(\log(\sqrt{S}))$   & $\Theta(1)$       & $\Theta(1)$     
  \\ 
  \bottomrule
  \end{tabular}
}
\caption{Sample complexity (\textbf{SC}) of stochastic bilevel optimization methods for finding an $\epsilon$-stationary point of Problem~\eqref{mainprobstoch} with LL of type~\eqref{minminstochprob}. \textbf{BS-LL} is the LL mini-batch size, i.e. the one used to approximate $\Phi$ in the LL solver. \textbf{WS} indicates the use of warm-start, e.g. Y, N means that warm-start is used for the LL problem but not for the LS. $t_\s$ and  $k_\s$ denote the number of iterations for the LL and LS problems respectively, while $\alpha_\s$ and $\sz{t,\s}$ are the stepsize respectively for the UL and LL problems at the $\s$-th UL iteration and $t$-th LL iteration. $L_f$ is the Lipschitz constant of $\nabla f$, $\stot$ is the total number of UL iteration and ESI means that the LS estimator is given by an exact single sample LL hessian inverse. The last 7 results are obtained under additional expected smoothness assumptions \citep{arjevani2019lower}.
}
    \label{tab:1}
\end{table*}
Bilevel optimization has a long history, see  \citep{dempe2020bilevel} for a comprehensive review. In this section we only present results which are closely related to ours.

Several gradient-based algorithms, together with sample complexity rates have been recently introduced for stochastic bilevel problems with LL of type~\eqref{minminstochprob}. They all follow a structure similar to \Cref{algo2}, where each UL update uses one (or more for variance reduction methods) hypergradient estimator computed using a variant of \Cref{algo1} with different LL and LS solvers. The algorithms mainly differ in how they compute the LL, LS and UL updates (e.g. in the choice of the step sizes $\eta_{t,\s}, \alpha_s$, mini-batch sizes, and whether they use variance reduction techniques), in the number of LL and LS iterations $t_s$, $k_s$, and in the use of warm-start. These differences are summarized in \Cref{tab:1}.

\cite{ghadimi2018approximation} introduce the first convergence analysis for a simple double-loop procedure, both in the deterministic and stochastic settings. Their algorithm uses (stochastic) gradient descent both at the upper and lower levels (SGD-SGD) and approximates the LS solution using an estimator of the inverted LL hessian based on truncated Neumann series (with $k_s$ elements). In the stochastic setting, this procedure needs $O(\epsilon^{-3})$ samples to reach an $\epsilon$-stationary point. This sample complexity is achieved by increasing the number of LL and LS iterations, i.e. at the $s$-th UL iteration it sets $t_s = \Theta(\sqrt{s})$ and $k_s = \Theta(\log(\sqrt{s}))$.

Differently from this seminal work, all subsequent ones warm-start the LL problem to improve the sample complexity, since this allows them to choose $t_s = \Theta(1)$ or even $t_\s=1$, the latter case is also referred to as \textit{single-loop}. 
Warm-start combined with the simple SGD-SGD strategy can improve the sample complexity by carefully selecting the UL and LL stepsize, i.e. using two timescale \citep{hong2020two} or single timescale \citep{chen2021tighter} stepsizes, or by employing larger and  $\epsilon$-dependent mini-batches \citep{ji2021bilevel}. Warm-starting also the LS can further improve the sample-complexity to $O(\epsilon^{-2})$ \citep{arbel2021amortized}. The complexity
$O(\epsilon^{-2})$ is optimal, since the optimal sample complexity of methods using unbiased stochastic gradient oracles with bounded variance on smooth functions is $\Omega(\epsilon^{-2})$, and this lower bound is also valid for bilevel problems of type~\eqref{mainprobstoch}\footnote{We can easily see this when $\fo(w,\lambda) = g(\lambda)$ and $\hat \fo(w,\lambda, \xi) = \hat g(\lambda, \xi)$ where $g:\Lambda \mapsto \R$ is Lipschitz smooth and $\hat g$ is an unbiased estimate of $g$ whose gradient w.r.t. $\lambda$ has bounded variance.} (also with LL of type~\eqref{minminstochprob}).

\cite{chen2021single,khanduri2021near,guo2021randomized,feihu2021biadam,yang2021provably} achieve the best-known sample complexity of $\tilde{O}(\epsilon^{-1.5})$ using variance reduction techniques\footnote{\cite{chen2021single} uses variance reduction only on the LL Hessian updates (see eq. (12)).}. \cite{li2021fully} introduce the first fully single loop algorithm where both the LL and LS are warm-started and solved with one iteration, although it achieves a sample complexity of $O(\epsilon^{-2})$ while using variance reduction. Variance reduction techniques  require additional algorithmic parameters and need expected smoothness assumptions to guarantee convergence \citep{arjevani2019lower}. Furthermore, they increase the cost per iteration compared to the SGD-SGD strategy since they require two stochastic samples per iteration to estimate gradients instead of one. For these reasons, we do not investigate these kinds of techniques in the present work. 

Except for \cite{chen2021single,guo2021randomized}, all aforementioned methods and ours are also computationally efficient, since they only require gradients and Hessian-vector products. Hessian-vector products have a cost comparable to gradients thanks to automatic differentiation. \cite{chen2021single,guo2021randomized} further rely on operations like inversions and projections of the LL Hessian. These can be too costly with a large number ($d$) of LL variables, which can make it impractical even to compute the full hessian.

All the aforementioned works study smooth bilevel problems with LL of type~\eqref{minminstochprob} and with a twice differentiable and strongly convex LL objective. At last, we mention two lines of work which consider different bilevel formulations: \citep{bertrand2020implicit,bertrand2021implicit}, which study the error of hypergradient approximation methods for certain non-smooth bilevel problems, and \citep{liu2020generic,liu2022general,arbel2022non}, which analyse algorithms to tackle bilevel problems with more than one LL solution.

The sample complexity improvement that our method achieves compared to \cite{ghadimi2018approximation}, i.e. from $O(\epsilon^{-3})$ to $O(\epsilon^{-2})$, is possible because our  hypergradient estimator (SID) uses mini-batches of size $\Theta(\epsilon^{-1})$ (instead of $\Theta(1)$) to estimate $\nabla E$ and $\jac_2 \Phi$ and a stochastic solver with decreasing step-sizes (instead of the truncated Neumann series inverse estimator) also to solve the LS problem (similar to the LL solver).
This allows SID to have $O(\epsilon^{-1})$ mean squared error (see \Cref{cor:oneovertrate}). In contrast, the hypergradient estimator in \cite{ghadimi2018approximation} achieves $O(\epsilon^{-1})$ only for the bias, while the variance does not vanish. 
Consequently, we can use a more aggressive UL step-size (constant instead of decreasing), which reduces the number of UL iterations from $O(\epsilon^{-2})$ to $O(\epsilon^{-1})$.

Among the methods using warm-start, \textit{Amigo} \citep{arbel2021amortized} is the most similar to ours.  Indeed, it achieves the same $O(\epsilon^{-2})$ optimal sample complexity as BSGM. Also,
the number of UL iterations and the size of the mini-batch to estimate $\nabla E$ and $\jac_2 \Phi$ is  $O(\epsilon^{-1})$, as for our method. 
The main differences with respect to BSGM are in the use of (i) the warm-start procedure in the LL and LS problems, which in general decreases the complexity, (ii) mini-batch sizes of the order of $\Theta(\epsilon^{-1})$ to estimate $\Phi$ (in the LL), $\jac_1 \Phi $ (in the LS), which increase the complexity, contrasting with our choice of taking just one sample for estimating the same quantities. Overall, (i)-(ii) balance out and ultimately give the same total complexity.

We note that our improvement over point (ii) is necessary to achieve the optimal sample complexity.
Indeed, if one istead carries out the analysis by using (ii), constant step-sizes for the LS and LL, and setting $k_\s, t_\s = \Theta(\log(S))$, only suboptimal complexity of $O(\epsilon^{-2}\log(\epsilon^{-1}))$ is achieved, because mini-batches of size  $\Theta(\epsilon^{-1})$ are used $2S(1+ \log(S))$ (instead of just $2S$) times in $S$ UL iterations.

For the deterministic case, we improve the rate of \cite{ghadimi2018approximation} from $O(\epsilon^{-5/4})$ to $O(\epsilon^{-1}\log(\epsilon^{-1}))$ by setting $t_s = \Theta(\kappa \log(s))$ (and also $k_\s$) instead of $t_s = \ceil{(s+1)^{1/4}/2}$, where $\kappa=(1-q)^{-1}$ and $q$ is the contraction constant defined in \Cref{ass:hypergrad}\ref{ass:contraction}. \cite{ji2021bilevel,arbel2021amortized} have an improved  complexity of  $O(\epsilon^{-1})$, obtained by using warm-start and setting $t_\s, k_s = \tilde{\Theta}(\kappa)$, where $\kappa$ is corresponds to the LL condition number.

Finally, note that warm-start makes it possible to set $t_\s$ and $k_\s$ with no dependence on $\epsilon$ both in the deterministic and stochastic settings, improving the sample complexity (by removing a log factor) in the former case. However, in the stochastic case the complexity does not improve because
solving the LL and LS problems cannot have lower complexity than $O(\epsilon^{-1})$, which is that of the sample mean estimation error. Such complexity is already achieved by our stochastic fixed-point iteration solvers with decreasing step-sizes and no warm-start.

\section{Assumptions and Preliminary Results}\label{se:preliminaries}
We hereby state the assumptions used for the analysis, discuss them and outline in a lemma some useful smoothness properties of the bilevel problem.

\begin{assumption}
\label{ass:hypergrad} 
The set $\Lambda \subseteq \R^m$ is
closed and convex and
the mappings $\Phi\colon \R^d \times \Lambda \to \R^d$
and $E\colon \R^d \times\Lambda \to \R$
are differentiable in an open set containing $\R^d \times \Lambda$. For every $\lambda \in \Lambda$:
\begin{enumerate}[label={\rm (\roman*)}]
\item\label{ass:contraction} $\Phi(\cdot,\lambda)$ is a  contraction,  i.e., $\norm{\jac_1 \Phi(w, \lambda)} \leq \q$ for some $\q<1$ and for all $w \in \R^d$.
\item\label{ass:rho}
$\norm{\jac_i\Phi (w(\lambda), \lambda) - \jac_i\Phi (w, \lambda)} \leq \nu_i \norm{w(\lambda) - w}$ for $i \in \{1,2\}$, $\forall w \in \R^d$.
\item\label{ass:lo}
$\norm{\grad_i\fo (w(\lambda), \lambda) - \grad_i\fo (w, \lambda)} \leq \mu_i \norm{w(\lambda) - w}$ for $i \in \{1,2\}$, $\forall w \in \R^d$.
\item\label{ass:lipE} $\fo(\cdot, \lambda)$ is Lipschitz 
cont.\@ on $\R^d$  with constant $\Bo$.
\end{enumerate}
\end{assumption}
\begin{assumption}\label{ass:add} 
Let $w_0:\Lambda \to \R^d$. For every $w^* \in  \{ w(\lambda) \,|\, \lambda \in \Lambda \}, \, \lambda \in \Lambda$:
\begin{enumerate}[label={\rm (\roman*)}] 
\item\label{ass:lola} $\nabla_1E(w^*, \cdot), \nabla_2E(w^*, \cdot)$ are Lipschitz cont.\@ on $\Lambda$ with constants $\Lola, \Lolb$ respectively.
\item\label{ass:lipwl} $\jac_1\Phi (w^*, \cdot)$, $\jac_2\Phi (w^*, \cdot)$ are Lipschitz cont.\@ on $\Lambda$ with constants $\lipw$, $\lipl$ respectively.
\item\label{ass:boundw} $\norm{w(\lambda) - w_0(\lambda)} \leq \boundw$ for some $\boundw \geq 0$. 
\item\label{ass:boundjac} $\norm{\jac_2 \Phi(w(\lambda), \lambda)} \leq \boundjac$ for some $\boundjac \geq 0$.
\end{enumerate}
\end{assumption}

\begin{assumption}\label{ass:alt}
The random variables $\zeta$ and $\xi$ take values in measurable spaces $\Xi$ and $Z$
and $\hat \Phi : \R^d \times \Lambda \times Z \mapsto \R^d $, $\hat \fo : \R^d \times \Lambda \times \Xi \mapsto \R$ 
are measurable functions, differentiable w.r.t. the first two
arguments in an open set containing $\R^d \times \Lambda$, and, for all 
$w \in \R^d$, $\lambda \in \Lambda$:
\begin{enumerate}[label={\rm (\roman*)}]
    \item\label{eq:expjacexptwo_2} 
    $\Exp{\hat \Phi(w, \lambda, \zeta)} {=} \Phi(w,\lambda)$, 
    $\Exp{\hat \fo(w, \lambda, \xi)} {=} \fo(w,\lambda)$ and we can exchange derivatives with expectations when taking derivatives on both sides.
    \item\label{eq:expjacexp_iii_2} $\Var{\hat \Phi(w, \lambda, \zeta)} \leq  \sigone + \sigtwo \norm{\Phi(w,\lambda) - w}^2$ for some  $\sigone, \sigtwo \geq 0$. 
    \item\label{eq:expjacexp_iv_2} $\Var{\jac_1 \hat\Phi(w,\lambda, \zeta)}
    \leq \mone$, $\Var{ \jac_2 \hat\Phi(w,\lambda, \zeta)}
    \leq \mtwo$  for some $\mone, \mtwo \geq 0$. 
    \item\label{eq:expgradexp_iv_2} $\Var{ \nabla_1 \hat\fo(w,\lambda, \xi)}
    \leq \monee$, $\Var{ \nabla_2 \hat\fo(w,\lambda, \xi)}
    \leq \mtwoe$ for some  $\monee, \mtwoe \geq 0$.
\end{enumerate}
\end{assumption}

 Assumptions~\ref{ass:hypergrad}, \ref{ass:add} and \ref{ass:alt} are similar to the ones in \citep{ghadimi2018approximation}
and subsequent works, but extended to the bilevel fixed point formulation and sometimes weakened. 
Assumptions~\ref{ass:hypergrad} and \ref{ass:alt} are sufficient to obtain meaningful upper bounds   on the mean square error of the \SID{} estimator (\Cref{algo1}), while Assumption~\ref{ass:add} enables us to derive the convergence rates of the bilevel procedure in \Cref{algo2}.
The deterministic case can be studied by setting, in Assumption~\ref{ass:alt},  $\sigone = \sigtwo = \mone = \mtwo = \monee = \mtwoe = 0$.
\begin{remark}\label{rem:assumptions} \ 
\begin{enumerate}[label={\rm (\roman*)}]
    \item Although the majority of recent works set $\Lambda = \R^m$, many bilevel problems satisfy the assumptions above only when $\Lambda \neq \R^m$. E.g.\@, when $\lambda$ is a scalar regularization parameter in the $LL$ objective and $\Phi$ is the gradient descent map, $\lambda$ has to be bounded from below away from zero for $\Phi(\cdot, \lambda)$ to always be a contraction
    (Assumption~\ref{ass:hypergrad}\ref{ass:contraction}). Also, when $\Lambda$ and $\{ w_0(\lambda) \,|\, \lambda \in \Lambda \}$ are bounded and closed, and Assumption~\ref{ass:hypergrad}\ref{ass:contraction} is satisfied, then \ref{ass:add}\ref{ass:boundw}\ref{ass:boundjac} are satisfied because $w(\cdot)$ is continuous in $\Lambda$. Our analysis directly considers the case $\Lambda \subseteq \R^m$, which includes the others.
    \item The Lipschitz assumption on $\fo$ (\ref{ass:hypergrad}\ref{ass:lipE}) is needed to upper bound $\norm{\nabla_1 \fo(w_t(\lambda), \lambda)}$. Otherwise, this is difficult to achieve since, in the stochastic setting, we have no control on the LL iterates $w_t(\lambda)$. This assumption can be relaxed in the deterministic case.
    \item Assumption~\ref{ass:add}\ref{ass:boundjac} is weaker than the one commonly used in related works, which requires the partial Jacobian $\jac_2 \Phi(w,\lambda)$ to be bounded uniformly on $\R^d\times \Lambda$.
    By contrast, we assume only the boundedness on the solution path
    $\{ (w(\lambda),\lambda) \,\vert\, \lambda \in \Lambda\}$. This allows to
    extend to scope of applicability of the method. For example, when $\lambda \in [\lambda_{min}, \lambda_{max}]$ is the $L_2$-regularization parameter multiplying $(1/2)\norm{w}^2$ in the LL objective, $\Phi$ is the gradient descent map and $w_0(\lambda) = 0$, then $\norm{\jac_2 \Phi (w, \lambda)} = \norm{w}$ which is unbounded, while $\norm{\jac_2 \Phi (w (\lambda), \lambda)} = \norm{w(\lambda)}$ is bounded since $w(\cdot)$ is differentiable (from \ref{ass:hypergrad}\ref{ass:contraction}) and therefore continuous in $[\lambda_{min}, \lambda_{max}]$ which is a bounded and closed set. 
    
    \item\label{rem:wsass} Assumption~\ref{ass:add}\ref{ass:boundw} uniformly bounds the distance of the LL solution $w(\lambda)$ from the starting point of the LL solver $w_0(\lambda)$. A similar assumption (with $w_0(\lambda)=0$) is stated implicitly also in \citep{ghadimi2018approximation} (See e.g. definition of $M$ in eq. (2.28)). \ref{ass:add}\ref{ass:boundw} is not needed when using warm-start (see also \Cref{rm:advws}), although it is satisfied when $\Lambda$ and $\{ w_0(\lambda) \,|\, \lambda \in \Lambda \}$ are bounded and closed and \ref{ass:hypergrad}\ref{ass:contraction} holds, but also in some cases where $\Lambda$ is unbounded. For example in meta-learning, when $\lambda$ is the bias in the LL regularization, i.e.\@  $\Lambda = \R^d$, $\Phi(w,\lambda) = (1-\eta\gamma)w -\eta \grad \mathcal{L}(w) + \eta\gamma\lambda$ with $\mathcal{L}$ $L$-smooth, $w_0(\lambda)=\lambda$ and $\eta > 0$ being the LL step-size, we have $w(\lambda) = \lambda - \gamma^{-1}\nabla \mathcal{L}(w(\lambda))$ which implies $\sup_{\lambda \in \R^d}\norm{w(\lambda)} = \infty$  while $\sup_{\lambda \in \R^d}\norm{w(\lambda) - w_0(\lambda)} \leq \gamma^{-1}L$.

    \item Assumption~\ref{ass:alt}\ref{eq:expjacexp_iii_2} is more general than the corresponding one in \citep{ghadimi2018approximation},   which is a bound on the variance on the LL gradient estimator recovered by setting $\sigtwo=0$ and $\hat \Phi(w, \lambda, \xi) = w - \nabla_1  \hat{\mathcal{L}}(w, \lambda, \xi)$ with $\nabla_1  \hat{\mathcal{L}}(w, \lambda, \xi)$ being an unbiased estimator of the LL gradient. Having $\sigtwo > 0$ allows the variance to grow away from the fixed point, which occurs for example when the unregularized loss in the LL Problem~\eqref{minminstochprob} is not Lipschitz (like for the square loss).
\end{enumerate}
\end{remark}

\begin{remark}
Variance reduction methods \citep{chen2021single,guo2021randomized,khanduri2021near,feihu2021biadam}  require also an \emph{expected smoothness assumption} on $\nabla \hat \fo$, $\hat \Phi$ and $\jac \hat \Phi$ (often satisfied in practice). See \citep{arjevani2019lower}. A random function $g (\cdot, \xi)$, where $\xi$ is the random variable, meets the expected smoothness assumption if $\Exp{\norm{g(x_1, \xi) - g(x_2, \xi)}}^2$ $\leq \tilde{L}_{g}^2 \norm{x_1 - x_2}^2$, for every $x_1,x_2$, where $\tilde{L}_{g} \geq 0$. 
\end{remark}

The existence of the hypergradient $\nabla f(\lambda)$ is guaranteed by the fact that $\Phi$ and $\fo$ are differentiable and that $\Phi(\cdot, \lambda)$ is a contraction (Assumption~\ref{ass:hypergrad}\ref{ass:contraction}). Furthermore, we have the following properties for the bilevel problem.

\begin{lemma}[\textbf{Smoothness properties of the bilevel problem}]\label{lm:lemmanew} If Assumptions~\ref{ass:hypergrad} \\ and \ref{ass:add}\ref{ass:lola}\ref{ass:lipwl}\ref{ass:boundjac} are satisfied, the following statements hold.

\begin{enumerate}[label={\rm (\roman*)}]
    \item\label{wprimelm} $\norm{w'(\lambda)} \leq L_w := \frac{\boundjac}{1-\q}$ for every $\lambda \in \Lambda$.
    \item\label{wprimelip} $w'(\cdot)$ is Lipschitz continuous with constant
    \begin{equation*}
        \lipwprime = \frac{\lipl}{1-\q} + \frac{\boundjac}{(1-\q)^2} \Big( \rhol + \lipw + \frac{\rhow\boundjac}{1-\q} \Big).
    \end{equation*}
    \item\label{lipgradient} $\nabla f(\cdot)$ is Lipschitz continuous with constant
\end{enumerate}
\begin{equation*}\label{eq:lf}
    L_f = \Lolb + \Bo\lipwprime + \frac{\boundjac}{1-\q}\Big(\Lol+\Lola + \frac{\Low\boundjac}{1-\q} \Big) .
\end{equation*}
\end{lemma}
The proof is in \Cref{pr:lemmanew}. See Lemma 2.2 in \cite{ghadimi2018approximation} for the special case of Problem~\eqref{mainprobstoch} with LL of type~\eqref{minminstochprob}.

\section{Convergence of \SID{}}\label{se:sid}
\begingroup
\renewcommand{\wstoc}[1][t]{w_{#1}(\lambda)}
\renewcommand{\vstocj}[2][k]{ v_{{#1}}({#2},\lambda)}

In this section, we fix $\lambda$ and provide an upper bound to the mean squared error of the hypergradient approximation:
\begin{equation}
    \label{eq:MSE}
    \MSE := \Exp{\norm{\hat \grad f(\lambda) - \grad f(\lambda)}^2},
\end{equation}
where $\hat \nabla f(\lambda)$  is given by \SID{} (\Cref{algo1}). 
In particular, we show that when the mini-batch size $J$ and the number of LL and LS iterations $t$ and $k$ tend to $\infty$, and the algorithms to solve the LL and LS problems converge in mean square error, then the mean square error of $\hat \nabla f (\lambda)$ tends to zero. Moreover, using the stochastic fixed-point iteration solvers in \eqref{eq:llsolver}-\eqref{eq:lssolver} with decreasing stepsizes and setting $t=k=\J$ we have $\MSE = O(1/t)$. 

This analysis is similar to the one of Algorithm~1 in \cite{grazzi2021convergence} Section 3 but with some crucial differences. First, this work considers the more challenging setting with stochasticity also in the UL objective. Second, Algorithm~1 in \cite{grazzi2021convergence} is a special case of \Cref{algo1} with $\J=1$, and letting $\J \to \infty$ is necessary to have an hypergradient estimator with zero MSE in the limit.

In the following, we first provide an analysis which is actually agnostic with respect to the specific solvers of the LL and LS problems. More specifically, according to Algorithm~\ref{algo1}
\begin{align*}
  \hat \grad f(\lambda) : =& \grad_2 \bar \fo_{\J}(\wstoc{}, \lambda) 
    + \jac_2  \bar\Phi_{\J}(\wstoc{},\lambda)^\top \vstocj{\wstoc{}}{}.
      \vspace{-.25truecm}
\end{align*}
where $\wstoc{}$ is the output of a $t$ steps stochastic 
algorithm that approximates the LL solution $w(\lambda)$ starting from $w_0(\lambda)$ and, for every $w$, $\vstocj{w}$ is the output of a $k$ steps stochastic 
algorithm that approximates the solution $\voptj{w}$ of the linear system
\begin{equation*}
\label{eq:LSgen}
     (I - \jac_1 \Phi(w,\lambda)^\top) v = \nabla_1\bar{\fo}_J(w,\lambda).
\end{equation*}
Recall that $ \nabla_{i} \bar \fo_{\J}(\wstoc{}, \lambda) = \frac{1}{\J}\sum_{j=1}^{{\J}} \nabla_i \hat\fo(\wstoc{}, \lambda, \xi_j) $ for $i \in \{1, 2\}$ and $\jac_2 \bar\Phi_{\J}(\wstoc{},\lambda) = \frac{1}{\J}\sum_{j=1}^{{\J}} \jac_2 \hat\Phi(\wstoc{}, \lambda, \zeta'_j)$.  
To this respect we also make the following assumption.
\begin{assumption}\label{ass:innerbackrates}
For every $w \in \R^d$, $\lambda \in \Lambda$,
 $t,k, \J \geq 1$, $j \in \{1, \dots, \J\}$, the random variables  $\vstocj{w}$, $\wstoc$, $\zeta'_j$ are mutually independent, $\wstoc$ is independent of $\xi_j$ and
\begin{align*}
    \Exp{\norm{\wstoc -w(\lambda)}^2} &\leq \rf(t), \qquad
    \Exp{\norm{\vstocj{w} - \voptj{w}}^2 } \leq \hrf(k),
\end{align*}
where $\rf:\N \mapsto \R_{+}$ and $\hrf:\N \mapsto \R_{+}$.
\end{assumption}

To analyse the MSE in \eqref{eq:MSE}, 
we start with the standard bias-variance decomposition
\begin{equation}\label{eq:biasvar}
\begin{aligned}
    \MSE =& \underbrace{\norm{\Exp{\hat \grad f(\lambda)} - \grad f(\lambda)}^2}_{\text{bias}} +  \underbrace{\Var{\hat \grad f(\lambda)}}_{\text{variance}}.
\end{aligned} 
\end{equation}
Then, using the law of total variance, we can write the useful decomposition 
\begin{equation}\label{eq:20200601h}
    \Var{\hat \grad f(\lambda)}
    = \underbrace{\Exp{\Var{\hat \grad f(\lambda) \given \wstoc}}}_{\text{variance I}} + \underbrace{\Var{\Exp{\hat \grad f(\lambda) \given \wstoc}}}_{\text{variance II}}.
\end{equation}
In the following three theorems we will bound the bias and the variance terms of the MSE. After that we state the final MSE bound in \Cref{thm:finalbound}.

\begin{theorem}[\textbf{Bias upper bounds}]
\label{th:boundbias}
Suppose that Assumptions~\ref{ass:hypergrad},\ref{ass:alt}, \ref{ass:add}\ref{ass:boundjac} and \ref{ass:innerbackrates}
are satisfied. Let $\lambda \in \Lambda$, $t,k \in \N$. Let $ \hat \Delta_w := \norm{\wstoc - w(\lambda)} $, then the following hold.

\begin{enumerate}[label={\rm(\roman*)}]
\item\label{th:boundbias_i} 
$\norm[\big]{ \Exp{\hat \grad f(\lambda) \given \wstoc}   - \grad f(\lambda)}$ 
$\leq c_1\hat \Delta_w
+ \LPhi \sqrt{\hrf(k)} + \rhol\hat \Delta_w \sqrt{\hrf(k)}$.
\item\label{th:boundbias_ii} 
$\norm{\EE[\hat \grad f(\lambda)] - \grad f(\lambda)}$
$\leq c_1 \sqrt{\rf(t)} + \LPhi  \sqrt{\hrf(k)} + \rhol \sqrt{\rf(t)}\sqrt{\hrf(k)},$
\end{enumerate}
where
\begin{equation*}
    c_1 = \Lol + \frac{\Low\LPhi + \rhol\Bo}{1-\q} + \frac{\rhow\Bo\LPhi}{(1-\q)^2}.
\end{equation*}
\end{theorem}
The proof is 
in \Cref{pf:bias} and  similar to that of Theorem 3.1 in \cite{grazzi2021convergence}.

\begin{theorem}[\textbf{Variance I bound}]
\label{th:varboundone}
Suppose that Assumptions~\ref{ass:hypergrad},\ref{ass:alt}, \ref{ass:add}\ref{ass:boundjac} and \ref{ass:innerbackrates}
are satisfied. Let $\lambda \in \Lambda$, $t,k \in \N$.
Then 
\begin{equation*}
\begin{aligned}
    \Exp{\Var{\hat \grad f(\lambda) \given \wstoc}} \leq &\left(\mtwoe + 4\frac{\mtwo(\Bo^2 + \monee) + \LPhi^2\monee}{(1-\q)^2}\right) \frac{2}{\J} 
+ 8(\LPhi^2+ \mtwo)\hrf(k) 
\\&+ 8 \rhol^2 \rf(t)\left(\hrf(k) + \frac{\monee}{J(1-\q)^2} \right).
\end{aligned}
\end{equation*}
\end{theorem}
The proof is in \Cref{pf:varboundone}.

\begin{theorem}[\textbf{Variance II bound}]
\label{th:varboundtwo}
Suppose that Assumptions~\ref{ass:hypergrad},\ref{ass:alt}, \ref{ass:add}\ref{ass:boundjac} and \ref{ass:innerbackrates}
are satisfied. Let $\lambda \in \Lambda$, and $t,k \in \N$. Then 
\begin{equation*}
    \Var{\Exp{\hat \grad f(\lambda) \given \wstoc}} \leq 
    3 \big( c_1^2\rf(t) +\LPhi^2\hrf(k)
    + \rhol^2 \rf(t)\hrf(k) \big),
\end{equation*}
where $c_1$ is defined as in \Cref{th:boundbias}.
\end{theorem}
\begin{proof}
From the property of the variance (\Cref{lem:varprop}\ref{lem:varprop_ii}) we get 
$
    \Var{\Exp{\hat \grad f(\lambda) \given \wstoc}} 
\leq \Exp[\big]{\norm{\Exp{\hat \grad f(\lambda)\given \wstoc} - \grad f(\lambda)}^2}
$.
The statement follows from \Cref{th:boundbias}\ref{th:boundbias_i}, the inequality
$(a+b+c)^2 \leq 3 (a^2 + b^2 + c^2)$, then taking the total expectation and finally using \Cref{ass:innerbackrates}.
\end{proof}

\begin{theorem}[\textbf{MSE bound for \SID{}}]
\label{thm:finalbound}
Suppose that Assumptions~\ref{ass:hypergrad},\ref{ass:alt}, \ref{ass:add}\ref{ass:boundjac} and \ref{ass:innerbackrates}
are satisfied. Let $\lambda \in \Lambda$, and $t,k,\J \in \N$. Then, if we use Algorithm~\ref{algo1}, we have
\begin{equation*}\label{eq:finalbound}
\begin{aligned}
    \MSE \leq& \left(\mtwoe + 4\frac{\mtwo(\Bo^2  + \monee)  + \LPhi^2\monee}{(1-\q)^2}\right)\frac{2}{\J}  
     + \left(6 c^2_1 + \frac{8\rhol^2\monee}{(1-\q)^2}\right) \rf(t) \\ & + \left(14\LPhi^2 + 8\mtwo\right) \hrf(k)
    + 14\rhol^2 \rf(t)\hrf(k),
\end{aligned}
\end{equation*}
where $c_{1}$ is defined in Theorem~\ref{th:boundbias}.
In particular, if $\lim_{t\to \infty}\rf(t) = \lim_{k\to \infty}\hrf(k) = 0$, then
\begin{equation*}\label{eq:asmsebound}
    \lim_{t,k, \J \to \infty}  \MSE =  0 
\end{equation*}
\end{theorem}
\begin{proof}
Follows from \eqref{eq:biasvar}-\eqref{eq:20200601h} and summing bounds in Theorems~\ref{th:boundbias}\ref{th:boundbias_ii}, \ref{th:varboundone}, and \ref{th:varboundtwo}. 
\end{proof}

We will show in \Cref{se:llls} that by using the LL and LS solvers in \eqref{eq:llsolver}-\eqref{eq:lssolver} with carefully chosen decreasing stepsizes, we have $\rf(t) = O(1/t)$ and $\hrf(k) = O(1/k)$ and hence, by setting $t=k=J$ we can achieve $\MSE = O(1/t)$ (\Cref{cor:oneovertrate}).

\subsection{Convergence of Solvers for The Lower-Level Problem and Linear System}\label{se:llls}
We analyse the convergence of a stochastic version of the Krasnoselskii-Mann iteration for contractive operators used in \Cref{algo1} to solve both LL and LS problems. A similar analysis is done in \cite[Section~5]{grazzi2021convergence}.

We recall the procedures \eqref{eq:llsolver}, \eqref{eq:lssolver} used to solve the LL and LS problems in \Cref{algo2}. Let $\zeta$, $\xi$ be random variables with values in $Z$ and $\Xi$. Let $(\zeta_t)_{t \in \N}$ and $(\hat \zeta_t)_{t \in \N}$ be independent copies of $\zeta$ and let $(\sz{t})_{t\in \N}$ be a sequence of stepsizes.

For every $w \in \R^d$ we let $\vstocj[0]{w} = 0$, $w_0:\Lambda \to \R^d$ satisfying Assumption~\ref{ass:add}\ref{ass:boundw}, and, for $k,t \in \N$,
\begin{align}
    \wstoc[t+1] &: = \wstoc[t] + \sz{t} (\hat \Phi(\wstoc[t], \lambda, \zeta_{t}) - \wstoc[t]), \label{eq:inneralg} \\
    \vstocj[k+1]{w} &: = \vstocj[k]{w}  + \sz{k} (\hat \Psi_w(\vstocj[k]{w}, \lambda, \hat \zeta_{k}) - \vstocj[k]{w}), \label{eq:backalg}
\end{align}
where $\hat \Psi_w(v, \lambda, z) : = \jac_1 \hat \Phi (w, \lambda, z)^\top v + \grad_1 \bar \fo_J (w, \lambda)$ and  $\bar \fo_\J (w, \lambda) = (1/\J) \sum_{j=1}^{J} \hat \fo (w, \lambda, \xi_j) $, $(\xi_j)_{1 \leq j \leq \J}$ being i.i.d. copies of the random variable $\xi \in \Xi$.

Note that to reduce the number of hyperparameters of the method, we use the same sequence of stepsizes $(\sz{t})_{t\in \N}$ for both the LL and LS problems. This choice might not be optimal and results in more conservative step sizes.

\begin{theorem}
\label{th:firstconv}
Let \Cref{ass:hypergrad}\ref{ass:contraction}, \ref{ass:alt} and \ref{ass:add}\ref{ass:boundw} hold. Let $\wstoc$ and $\vstocj{w}$ be defined as in \eqref{eq:inneralg} and \eqref{eq:backalg}. 
Assume $\sum_{t=0}^\infty \sz{t} = +\infty$ and $\sum_{t=0}^\infty \sz{t}^2 < +\infty$. Then, for every $\lambda \in \Lambda$, $w \in \R^d$, we have
\begin{equation*}
        \lim_{t \to \infty} \wstoc = w(\lambda), \quad
        \lim_{k \to \infty}\vstocj{w} = \voptj{w} \quad \PP\text{-a.s.}
\end{equation*}
Moreover, let $\tilde{\sigma}_2 := \max\{ 2\mone/(1-q)^2, \sigtwo \}$ and $\sz{t}:= \beta/(\gamma + t)$ with 
$\beta > 1/(1-\q^2)$ and $\gamma \geq \beta(1+\tilde{\sigma}_2)$.
Then for every $w \in \R^d$, $t,k > 0$ 
\begin{align}
    \label{eq:ratell}  \Exp{\norm{\wstoc - w(\lambda)}^2} \leq \frac{ d_{w}}{\gamma + t} \qquad
    \Exp{\norm{\vstocj{w} - \voptj{w}}^2} \leq \frac{d_{v}}{\gamma + k} 
\end{align}
where 
\begin{align*}
     d_{w} &:= \max \left\{\gamma \boundw^2, \frac{\beta^2 \sigone}{\beta(1-\q^2) - 1}\right\}, \\
    d_{v} &:= \max \left\{\frac{\Bo^2 +\monee}{(1-\q)^2}\gamma ,\frac{2(\Bo^2 +\monee)\mone}{(1-\q)^2} \frac{\beta^2}{\beta(1-\q^2) - 1}\right\} %
\end{align*}
Alternatively, with constant step size $\eta_t = \eta \leq 1/(1+\tilde{\sigma}_2)$
\begin{align}
    \label{eq:ratell_ii}  \Exp{\norm{\wstoc - w(\lambda)}^2} &\leq  (1- \eta (1 - \q^2))^t\boundw^2 + \frac{\eta \sigone}{1-\q^2}\\
    \Exp{\norm{\vstocj{w} - \voptj{w}}^2} &\leq (1- \eta (1 - \q^2))^k \frac{\Bo^2 +\monee}{(1-\q)^2} + \frac{\eta}{1-\q^2} \frac{2(\Bo^2 +\monee)\mone}{(1-\q)^2} \label{eq:ratels_ii}
\end{align}
\end{theorem}
\begin{proof}
The statement follows by applying Theorems 4.1 and 4.2 in 
\citep{grazzi2021convergence} with 
 $\hat{T} =\hat \Phi(\cdot, \lambda, \cdot)$ and $\hat{T} =\hat \Psi_w(\cdot, \lambda, \cdot)$ where we recall that $\hat \Psi_w(v, \lambda, z)  = \jac_1 \hat \Phi (w, \lambda, z)^\top v + \grad_1 \bar \fo_J (w, \lambda)$ and  $\bar \fo_J (w, \lambda) = (1/\J) \sum_{j=1}^{J} \hat \fo (w, \lambda, \xi_j) $, $(\xi_j)_{1 \leq j \leq \J}$ being i.i.d. copies of the random variable $\xi \in \Xi$.
To that purpose, in view of those theorems it is sufficient to verify 
Assumptions~D  in \citep{grazzi2021convergence}.
This is immediate for $\hat \Phi(\cdot, \lambda, \cdot)$,
due to Assumptions~\ref{ass:hypergrad}\ref{ass:contraction} and \ref{ass:alt}. Further, applying \ref{ass:add}\ref{ass:boundw} and \ref{ass:alt}\ref{eq:expjacexp_iii_2}  gives the first inequality in \eqref{eq:ratell} and \eqref{eq:ratell_ii}.
Concerning $\hat \Psi_w(\cdot, \lambda, \cdot)$, let  $\Exptilde{\cdot} = \Exp{\cdot \given (\xi_j)_{1 \leq j \leq \J}}$ and  $\Vartilde{\cdot} = \Var{\cdot \given (\xi_j)_{1 \leq j \leq \J}}$. It follows from Assumptions~\ref{ass:hypergrad}\ref{ass:contraction} and \ref{ass:alt}\ref{eq:expjacexptwo_2} that
\begin{align*}
    \Exptilde{\hat \Psi_w(v, \lambda, \zeta)} &=  \jac_1 \Phi(w, \lambda)^\top v + \grad_1 \bar E_J(w,\lambda) =: \Psi_w(v,\lambda).
\end{align*}
Since $\norm{\jac_1 \Psi_w(v,\lambda)} = \norm{ \jac_1 \Phi(w, \lambda)} \leq \q$, $\Psi_w(\cdot,\lambda)$ is a contraction with constant $\q$ and  Assumption~D(i)-(ii) in \citep{grazzi2021convergence} are satisfied.
Furthermore, from Assumption~\ref{ass:alt}
\begin{equation}\label{eq:partt1}
        \Vartilde{\hat\Psi_w(v, \lambda, \zeta)} \leq \norm{v}^2 \mone,
\end{equation}
and 
\begin{align*}
    \norm{v} &\leq \norm{\Psi_w(v,\lambda) - v} + \norm{\Psi_w(v,\lambda) } \\
     &\leq \norm{\Psi_w(v,\lambda) - v} + \norm{\jac_1 \Phi(w,\lambda)^\top v + \nabla_1 \bar\fo_J(w, \lambda)}\\
     &\leq \norm{\Psi_w(v,\lambda) - v} + \q\norm{v} + \norm{\nabla_1 \bar\fo_J(w, \lambda)}.
\end{align*}
It follows that
\begin{equation}\label{eq:partt2}
    \norm{v} \leq \frac{1}{1-\q} \left( \norm{\Psi_w(v,\lambda) - v} + \norm{\nabla_1 \bar\fo_J(w, \lambda)} \right).
\end{equation}
Hence, combining \eqref{eq:partt1} and \eqref{eq:partt2} we obtain
\begin{equation*}
     \Vartilde{\Psi_w(v, \lambda, \zeta)} \leq \frac{2\mone}{(1-\q)^2} \norm{\Psi_w(v,\lambda) - v}^2  +\frac{2\norm{\nabla_1 \bar\fo_J(w, \lambda)}^2\mone}{(1-\q)^2},
\end{equation*}
which satisfies Assumption D(iii) in \citep{grazzi2021convergence}. Thus, we can apply Theorem 4.1 and 4.2 in \citep{grazzi2021convergence} to obtain results on $\vstocj{w}$ which  hold conditioned to $(\xi_j)_{j=1}^{J}$. The bounds in the second inequality of \eqref{eq:ratell} and in \eqref{eq:ratels_ii} are finally obtained by taking the total expectation and noting that
\begin{equation*}
    \Exp{\norm{\nabla_1 \bar\fo_J(w, \lambda)}^2} = \norm{\nabla_1 \fo(w, \lambda)}^2 + \Var{\nabla_1 \bar\fo_J(w, \lambda)}  \leq \Bo^2 + \monee/J 
    \leq \Bo^2 + \monee.
\end{equation*}
\end{proof}

\begin{remark}[On warm-start]
Using Assumption~\ref{ass:add}\ref{ass:boundw} and setting $\vstocj[0]{w} = 0$ we removed any dependency on the starting points for the LL and LS in the final rates of \Cref{th:firstconv}. 
On the contrary, previous work have exploited this dependency to study the warm-start of the LL (LS) which sets, at the $\s$-th UL iteration $w_0(\lambda_\s) = w_{t}(\lambda_{\s-1})$ ($v_0(w, \lambda_\s) = v_{k}(w, \lambda_{\s-1})$). 
However, this  complicates the analysis, since the rates of \Cref{th:firstconv} will also depend on the UL update (e.g. on the UL step size $\alpha_\s$).
\end{remark}

\endgroup

\begin{corollary}\label{cor:oneovertrate} Suppose that Assumptions~\ref{ass:hypergrad},\ref{ass:alt}, \ref{ass:add}\ref{ass:boundjac}\ref{ass:boundw}
are satisfied and suppose that $\hat \nabla f(\lambda)$ is computed via  Algorithm~\ref{algo1} with $t=k=\J \in \N$ and LL/LS stepsizes $(\eta_j)_{j\in\N}$ chosen according to the decreasing case of \Cref{th:firstconv}.  Then, we obtain
\begin{equation}
\label{eq:finalbound2}
\begin{gathered}
    \MSE \leq \frac{c_{b} + c_{v}}{t},
\end{gathered}
\end{equation}
 where
\begin{equation}
    \begin{aligned}
    c_b &= 3c^2_{1} d_w + 3\LPhi^2 d_v  + 3 \rhol^2 d_w d_v\\
    c_{v} &= \mtwoe + 8\frac{\mtwo(\Bo^2  + \monee) + \LPhi^2\monee}{(1-\q)^2} + \left(3c^2_{1} + \frac{8\rhol^2\monee}{(1-\q)^2}\right) d_w \\  &\quad+ (11 \LPhi^2 + 8\mtwo) d_v + 11\rhol^2 d_v d_w,
    \end{aligned}
\end{equation}
and $d_w$, $d_v$, $c_1$ are defined in  \Cref{th:firstconv,th:boundbias}. Hence, $\MSE \leq \epsilon$ in $t = O(\epsilon^{-1})$.
\end{corollary}

\section{Convergence of BSGM}\label{se:bilevelconv}

In this section, we first derive convergence rates of the projected inexact gradient method for $L$-smooth possibly non-convex objectives (\Cref{se:PISGM}).
Then, we combine this result with the mean square error upper bounds in \Cref{se:sid} to obtain in \Cref{se:bilevelrates}, the desired convergence rate and sample complexity for BSGM (\Cref{algo2}).

\subsection{Projected Inexact Gradient Method}\label{se:PISGM}

Let $f:\D \mapsto \R$, be an $L$-smooth function  on the convex set $\D \subseteq \R^m$. We consider the following \textit{projected inexact gradient descent} algorithm
\begin{equation}\label{eq:pgdinexact}
\begin{array}{l}
    \lambda_0 \in \D \\
    \text{for}\; \s=0,1,\ldots\\[0.4ex]
    \left\lfloor
    \begin{array}{l}
    \lambda_{\s+1} = \proj\left(\lambda_\s - \alpha \g{\s}\right), \qquad 
    \end{array}    
    \right.
\end{array}
\end{equation}
where $\proj$ is the projection onto $\D$, $\alpha > 0$ is the step-size and $\g{\s}$ is s stochastic  estimator of the gradient. We stress that we do not assume that $\g{\s}$ is unbiased. 

\begin{definition}[\textbf{Proximal Gradient Mapping}] The proximal gradient mapping of $f$ is
\begin{equation*}\label{eq:proxgradmap}
G_\alpha(\lambda) : = \alpha^{-1} \left(\lambda - \proj(\lambda - \alpha \nabla f(\lambda)) \right).
\end{equation*}
\end{definition}

The above gradient mapping is commonly used in constrained non-convex optimization as a replacement of the gradient for the characterization of stationary points (see e.g. \citep{drusvyatskiy2018error}). Indeed, $\lambda^*$ is a stationary point if and only if $G_\alpha(\lambda^*) = 0$ and in the unconstrained case (i.e. $\Lambda = \R^m$) we have $G_\alpha(\lambda) = \nabla f(\lambda)$.
Since the algorithm is stochastic we provide guarantees in expectation. In particular, we bound  
$
\frac{1}{\stot}\sum_{\s=0}^{\stot-1}  \Exp{\norm{G_\alpha(\lambda_\s)}^2}
$.
Note that this quantity is always greater than or equal to $\min_{\s \in \{0,\dots \stot-1\}}  \Exp{\norm{G_\alpha(\lambda_\s)}^2}$, meaning that at least one of the iterates satisfies the bound.

The following theorem and subsequent corollary provide such upper bounds which have a linear dependence on the average MSE of $\g{\s}$. A similar setting is studied also by \cite{dvurechensky2017gradient} where  they consider inexact gradients but with a different error model. \cite{schmidt2011convergence} provide a similar result in the convex case.

\begin{theorem}\label{lm:constrained_prev}
Let $\D \subseteq \R^m$ be convex and closed, $f:\D \mapsto \R$ be $L$-smooth and $\{\lambda_\s\}_\s$ be a sequence generated by Algorithm~\eqref{eq:pgdinexact}. Furthermore, let $\Delta_f := f(\lambda_0)- \inf_{\lambda} f(\lambda)$, $\yd > 0$, $\delta_s := \norm{\nabla f(\lambda_s) - \hat \nabla f(\lambda_s)}$ and  $0 < \alpha < 2/[L(1 + \yd)]$.
Then for all $S \in \N$
\begin{equation*}\label{eq:ratesquaredsummable_prev}
    \frac{1}{\stot}\sum_{\s=0}^{\stot-1}  \norm{G_\alpha(\lambda_\s)}^2 \leq  \frac{1}{\stot}\left[\frac{4\Delta_f}{c_\alpha L(1 + \yd))} + 2 \left(1 + \frac{1}{ c_\alpha L \yd}\right) \sum_{\s=0}^{\stot-1}\delta^2_s \right],
\end{equation*}
where $c_\alpha = \alpha(2-\alpha L(1 + \yd))$.

\end{theorem}
\begin{proof}
Since $\D$ is convex and closed, the projection is a \textit{firmly non-expansive} operator, i.e. for every $\gamma, \beta \in \R^n$,
\begin{align*}
    \norm{\proj(\gamma) - \proj(\beta)}^2 + \norm{\gamma - \proj(\gamma) -\beta + \proj(\beta)}^2  &\leq \norm{\gamma - \beta}^2,
\end{align*}
which yields, by expanding the second term in the LHS
\begin{align*}
        2\norm{\proj(\gamma) - \proj(\beta)}^2 + \norm{\gamma  -\beta}^2 - 2(\gamma - \beta)^\top (\proj(\gamma) - \proj(\beta))  &\leq \norm{\gamma - \beta}^2,
\end{align*}
and, after simplifying
\begin{equation*}
   \norm{\proj(\gamma) - \proj(\beta)}^2 \leq  (\gamma - \beta)^\top (\proj(\gamma) - \proj(\beta)).
\end{equation*}
In particular, substituting $\gamma = \lambda_\s$ and $\beta= \lambda_\s -\alpha \g{\s}$ we get
\begin{equation}\label{eq:fnonexpansive}
    \norm{\lambda_\s - \lambda_{\s+1}}^2 \leq \alpha \g{\s}^\top (\lambda_\s - \lambda_{\s+1}).
\end{equation}
Now, it follows from the Lipschitz smoothness of $f$ that for every $\gamma,\beta \in \D$
\begin{equation*}\label{eq:smoothprop}
    f(\beta) \leq f(\gamma) + \nabla f(\gamma)^\top (\beta -\gamma) + \frac{L}{2} \norm{\beta - \gamma}.
\end{equation*}
Then substituting $\gamma =\lambda_\s$ and $\beta=\lambda_{\s+1}$, and letting $c' = L c$ with $c > 0$, we obtain
\begin{align*}
    f(\lambda_{\s+1}) &\leq f(\lambda_\s) - (\grad f(\lambda_\s) \mp \g{\s})^\top (\lambda_\s - \lambda_{\s+1}) + \frac{L}{2}\norm{\lambda_\s-\lambda_{\s+1}}^2 \\
    &\leq f(\lambda_\s) - (\grad f(\lambda_\s) - \g{\s})^\top (\lambda_\s - \lambda_{\s+1}) + \left(\frac{L}{2} - \frac{1}{\alpha}\right)\norm{\lambda_\s-\lambda_{\s+1}}^2 \\
    &\leq f(\lambda_\s) + \frac{1}{2\yc}\norm{\grad f(\lambda_\s) - \g{\s})}^2 + \left(\frac{L+\yc}{2} - \frac{1}{\alpha}\right)\norm{\lambda_\s-\lambda_{\s+1}}^2 \\
     &\leq  f(\lambda_\s) + \frac{1}{2 c'}\norm{\grad f(\lambda_\s) - \g{\s}}^2 -\eta\norm{\lambda_\s-\lambda_{\s+1}}^2,
\end{align*}
where we used \cref{eq:fnonexpansive} for the second line, the Young inequality $a^\top b \leq (1/2\yc)\norm{a}^2 + (\yc/2)\norm{b}^2$ in the third line, and the definition $\eta : = 1/\alpha -(L+\yc)/2$, which is positive due to the assumption on $\alpha$, in the last line.
Rearranging the terms we get
\begin{align}\label{eq:part1}
\norm{\lambda_\s-\lambda_{\s+1}}^2  & \leq \frac{1}{\eta}\left( f(\lambda_\s) - f(\lambda_{\s+1}) + \frac{1}{2\yc} \norm{\grad f(\lambda_\s) - \g{\s}}^2 \right).
\end{align}
Furthermore, let $\bar \lambda_\s : = \proj(\lambda_\s - \alpha \nabla f(\lambda_\s))$. Then,
we have that
\begin{equation}\label{eq:part2}
\begin{aligned}
\norm{\lambda_{\s+1}- \bar\lambda_{\s}}^2 &= \norm{  \proj(\lambda_\s - \alpha \g{\s}) -  \proj(\lambda_\s - \alpha \nabla f(\lambda_\s))}^2 \\
&\leq \alpha^2\norm{\g{\s} - \nabla f(\lambda_\s) }^2,
\end{aligned}
\end{equation}
where we used the fact that the projection is $1$-Lipschitz.

Now, recalling the definition of $G_\alpha(\lambda)$ we have that $G_\alpha(\lambda_\s) = \alpha^{-1} (\lambda_\s - \bar \lambda_\s)$ and hence, using the inequalities \eqref{eq:part1} and  \eqref{eq:part2}, we have
\begin{align*}
    \norm{G_\alpha(\lambda_\s)}^2 &= \alpha^{-2}\norm{\lambda_\s \mp \lambda_{\s+1} - \bar\lambda_\s}^2 \\
    &\leq 2\alpha^{-2}\left(\norm{\lambda_\s - \lambda_{\s+1}}^2 +  \norm{\lambda_{\s+1} - \bar\lambda_\s}^2\right) \\
  &\leq \frac{2}{\eta\alpha^2} \left(f(\lambda_\s) - f(\lambda_{\s+1}) + \frac{1}{2\yc}\norm{\g{\s} - \nabla f(\lambda_\s) }^2 \right) +  2\norm{\g{\s} - \nabla f(\lambda_\s) }^2 \\
    &= \frac{2}{\eta\alpha^2} \left( f(\lambda_\s) - f(\lambda_{\s+1}) \right) + \left(2 + (\eta\yc)^{-1}\alpha^{-2}\right)\norm{\g{\s} - \nabla f(\lambda_\s) }^2.
\end{align*}
Summing the inequalities over $\s$ and noting that $ - f(\lambda_s) \leq - \, \inf_{\lambda} f(\lambda)$ we get
\begin{align*}
    \sum_{\s=0}^{\stot-1}  \norm{G_\alpha(\lambda_\s)}^2 & \leq  \frac{2\Delta_f}{\eta\alpha^2} + \left(2 + (\eta\yc)^{-1}\alpha^{-2}\right) \sum_{\s=0}^{\stot-1} \norm{\g{\s} - \nabla f(\lambda_\s) }^2.
\end{align*}
Finally, dividing both sides of the above inequality by $\stot$, recalling the definition of $\eta$, $\delta_s$ and $c'$, \eqref{eq:ratesquaredsummable} follows.
\end{proof}

\begin{corollary}\label{lm:constrained}
Under the same assumptions of  \Cref{lm:constrained_prev}  we have
\begin{equation*}\label{eq:ratesquaredsummable}
    \frac{1}{\stot}\sum_{\s=0}^{\stot-1}  \Exp{\norm{G_\alpha(\lambda_\s)}^2} \leq  \frac{1}{\stot}\left[\frac{4\Delta_f}{c_\alpha L(1 + \yd))} + 2 \left(1 + \frac{1}{ c_\alpha L \yd}\right) \sum_{\s=0}^{\stot-1}\MSEC{\g{\s}} \right],
\end{equation*}
where $c_\alpha = \alpha(2-\alpha L(1 + \yd))$.
Consequently, setting $c = 1/2$, for any $\alpha \leq 1/L$  we have
\begin{equation*}\label{eq:ratesquaredsummablesimple}
    \frac{1}{\stot}\sum_{\s=0}^{\stot -1}  \Exp{\norm{G_\alpha(\lambda_\s)}^2} \leq  \frac{1}{\stot \alpha}\left[ 8\Delta_f + \frac{10}{L} \sum_{\s=0}^{\stot-1}\MSEC{\g{\s}} \right].
\end{equation*}
We recall that $\MSE : = \Exp{\norm{\hat \nabla  f(\lambda) -\nabla f(\lambda)}^2}$.
\end{corollary}
\begin{proof}
Follows by taking expectation of the inequality in the statement of \Cref{lm:constrained_prev}
\end{proof}
\begin{remark}
Note that if the error term $\sum_{\s=0}^{\stot-1}\MSEC{\g{\s}}$ grows sub-linearly with $S$, \Cref{lm:constrained} provides useful convergence rates. In particular, when $\sum_{\s=0}^{\infty}\MSEC{\g{\s}} < \infty$, we have a convergence rate of $O(1/\stot)$, which matches the optimal rate of (exact) gradient descent on smooth and possibly non-convex objectives.
\end{remark}

\subsection{Bilevel Convergence Rates and Sample Complexity}\label{se:bilevelrates}

Here, we finally prove the convergence rates and sample complexity of \Cref{algo2} by combining the results of the previous section with the bounds on the MSE of the hypergradient estimator obtained in \Cref{se:sid}.
\begin{definition}[\textbf{Sample Complexity}]
An algorithm which solves the stochastic bilevel problem in \eqref{mainprobstoch} has sample complexity $N$ if the total number of samples of $\zeta$ and $\xi$ is equal to $N$. For \Cref{algo2}, this corresponds to the total number of evaluations of $\nabla \hat E, \hat \Phi, \jac \hat \Phi^\top v$.
\end{definition}

In the following theorem we establish the sample complexity of \Cref{algo2} for $t_\s = k_\s = J_\s = \ceil{\csid(\s+1)}$ and $t_\s = k_\s = J_\s \ceil{\csid \stot}$ (finite horizon), where $\csid > 0$ is an additional hyperparameter that can be tuned empirically.

\begin{theorem}[\textbf{Stochastic BSGM}]\label{th:finalconv}
Suppose that $\Lambda \subseteq \R^m$ and Assumptions~\ref{ass:hypergrad}, \ref{ass:add}, \ref{ass:alt} are satisfied.
Assume that the bilevel Problem~\eqref{mainprobstoch} is solved by Algorithm~\ref{algo2} with $\alpha \leq 1/L_f$ and $(\eta_j)_{j \in \N}$ are decreasing and chosen according to \Cref{th:firstconv}, where $L_f$ is defined in \Cref{lm:lemmanew}. Let $\lambda_0 \in \Lambda$, $G_\alpha(\lambda) : = \alpha^{-1} \left(\lambda - \proj(\lambda - \alpha \nabla f(\lambda)) \right)$ be the proximal gradient mapping, $\csid > 0$, and $c_b$ and $c_v$ be the defined in \Cref{cor:oneovertrate}. Then the following hold.
\begin{enumerate}[label={\rm (\roman*)}]
    \item\label{resone} Suppose that for every $s \in \N$ $t_\s=k_\s =\J_\s = \ceil{\csid(\s+1)}$. Then for every $S \in \N$ we have
\begin{equation*}\label{eq:birate}
    \frac{1}{\stot}\sum_{\s=0}^{\stot -1}  \Exp{\norm{G_\alpha(\lambda_\s)}^2} \leq  \frac{1}{\stot \alpha}\left[ 8\Delta_f + \frac{10}{L_f} \frac{c_b + c_v}{\csid} (\log(\stot) + 1) \right].
\end{equation*}
Moreover, after $\tilde{O}(\epsilon^{-2})$  samples there exists $\s^* \leq S-1$  such that $ \Exp{\norm{G_\alpha (\lambda_{\s^*})}^2} \leq \epsilon$.
\item\label{restwo} Finite horizon. Let $\stot \in \N$, and suppose that for $s = 0,\dots,S-1$, $t_\s=k_\s =\J_\s = \ceil{\csid \stot}$. Then we have
\begin{equation*}\label{eq:biratetwo}
    \frac{1}{\stot}\sum_{\s=0}^{\stot -1}  \Exp{\norm{G_\alpha(\lambda_\s)}^2} \leq  \frac{1}{\stot \alpha}\left[ 8\Delta_f + \frac{10}{L_f} \frac{c_b + c_v}{\csid} \right].
\end{equation*}
Moreover, after $O(\epsilon^{-2})$  samples there exists $\s^* \leq S-1$  such that $ \Exp{\norm{G_\alpha (\lambda_{\s^*})}^2} \leq \epsilon$.
\end{enumerate}
\end{theorem}

\begin{proof}
We first compute $N$, i.e. the total number of samples used in $S$ iterations.
At the $\s$-th iteration, \Cref{algo2} requires executing \Cref{algo1} which uses  $t_\s + k_\s + J_s$ copies of $\zeta$, for evaluating $\hat \Phi$, $\jac_1 \hat \Phi^\top v$, and $\jac_2 \hat \Phi^\top v$, and additional $\J_\s$ copies of $\xi$ for evaluating  $\nabla \hat E$. 
Thus, the $s$-th UL iteration uses  $4\ceil{\csid (\s+1)}$ and $4\ceil{\csid S}$ samples for case \ref{resone} and \ref{restwo} respectively.
Hence, we have 
\begin{align*}
    \ref{resone}:& \quad 2\csid \stot^2 \leq N = 4  \sum_{\s=0}^{\stot-1}\ceil{\csid (\s+1)} \leq 4(\csid + 1) \stot^2. \\
     \ref{restwo}:& \quad4 \csid \stot^2 \leq N = 4 \ceil{\csid\stot} \sum_{\s=0}^{\stot-1} 1 \leq 4 (\csid+1) \stot^2.
\end{align*}
This implies that in both cases $N = \Theta(\stot^2)$ or equivalently $S = \Theta(\sqrt{N})$.

\ref{resone}:
 \Cref{cor:oneovertrate}, with $t_\s = \ceil{\csid (\s + 1)}$, yields
\begin{equation*}
    \sum_{\s=0}^{\stot-1} \MSEC{\g{\s}} \leq (c_b + c_v)\sum_{\s=0}^{\stot-1} \frac{1}{\csid(\s+1)}
    \leq \frac{c_b + c_v}{\csid} (\log(\stot) + 1).
\end{equation*}

Since $\nabla f$ is $L_f$-Lipschitz continuous, thanks to \Cref{lm:lemmanew} we can apply \Cref{lm:constrained} and obtain \ref{eq:birate}.
Therefore, we have $\frac{1}{\stot}\sum_{\s=0}^{\stot-1}  \Exp{\norm{G_\alpha(\lambda_\s)}^2} \leq \epsilon$ in a number of UL iterations  $S = \tilde{O}(\epsilon^{-1})$. Since we proved $N= \Theta(S^2)$, the sample complexity result for case \ref{resone} follows.

\ref{restwo}:
Similarly to the case \ref{resone}, we apply \Cref{cor:oneovertrate} with $t_\s = \ceil{\csid \stot}$ obtaining
\begin{equation*}
    \sum_{\s=0}^{\stot-1} \MSEC{\g{\s}} \leq (c_b + c_v)\sum_{\s=0}^{\stot-1} \frac{1}{\csid\stot} = \frac{c_b + c_v}{\csid}.
\end{equation*}
Since $\nabla f$ is $L_f$-Lipschitz, thanks to \Cref{lm:lemmanew}, we derive \ref{eq:biratetwo} from \Cref{lm:constrained}.

Therefore, in this case we have $\frac{1}{\stot}\sum_{\s=0}^{\stot-1}  \Exp{\norm{G_\alpha(\lambda_\s)}^2} \leq \epsilon$ in a number of UL iterations  $S = O(\epsilon^{-1})$. Since $N= \Theta(S^2)$, the sample complexity result for case \ref{restwo} follows.
\end{proof}

In the following theorem we derive rates for \Cref{algo2} in the deterministic case, i.e. when the variance of $\hat \Phi$ $\jac \hat \Phi$ and $\nabla \hat E$ is zero. In this case we will show that the LL and LS solvers in \Cref{algo1} can be implemented with constant step size and with $\J_\s = 1$, to obtain the near-optimal sample complexity of $\tilde{O}(\epsilon^{-1})$.

\begin{theorem}[\textbf{Deterministic BSGM}]\label{th:finalconvdet}
Suppose that $\Lambda \subseteq \R^m$ and Assumptions~\ref{ass:hypergrad}, \ref{ass:add}, \ref{ass:alt}
are satisfied with $\sigone = \sigtwo = \mone = \mtwo = \monee = \mtwoe = 0$, hence $\hat \Phi = \Phi$ and $\hat\fo = \fo$. Assume that the bilevel Problem~\eqref{mainprobstoch} is solved by Algorithm~\ref{algo2} with $\alpha \leq 1/L_f$ with $L_f$ defined in \Cref{lm:lemmanew}, $\eta_j = 1$, $t_\s=k_\s = \ceil{c_3 \log (\s + 1)}$ and $\J_\s = 1$,  and $c_3 \geq 1/\log(1/q) > 0$. Let $\lambda_0 \in \Lambda$ and $G_\alpha(\lambda) : = \alpha^{-1} \left(\lambda - \proj(\lambda - \alpha \nabla f(\lambda)) \right)$ be the proximal gradient mapping. 
Then 
\begin{equation*}\label{eq:biratedet}
    \frac{1}{\stot}\sum_{\s=0}^{\stot-1}  \norm{G_\alpha(\lambda_\s)}^2 \leq  \frac{1}{\stot \alpha}\left[8\Delta_f  + \frac{5C\pi^2}{3 L_f} \right],
\end{equation*}
where
\begin{equation*}
    C : = 3 \left(\Lol + \frac{\Low\LPhi + \rhol\Bo}{1-\q} + \frac{\rhow\Bo\LPhi}{(1-\q)^2} \right)^2\boundw^2 + 3\LPhi^2\frac{\Bo^2}{(1-\q)^2} + 3 \rhol^2 \frac{\boundw^2 \Bo^2}{(1-\q)^2}.
\end{equation*}
Also, after $O(\epsilon^{-1}\log(\epsilon^{-1}))$ samples there exists $\s^* \in \{0, \dots, \stot-1 \}$ such that $\norm{G (\lambda_{\s^*})}^2 \leq \epsilon$.
\end{theorem}
The Proof is in \Cref{proof:finalconvdet} and is similar to that of \Cref{th:finalconv}.

\begin{remark}[Dependency on the contraction constant\footnote{{\color{red} Corrected from published version.}}]\label{rem:cc}
By setting for the stochastic case $\eta_t = \beta/(\gamma + t)$ with $\beta = 2/(1-q^2)$ and $\gamma = \beta(1 + \tilde{\sigtwo})$ in \Cref{algo1}  and $\alpha = 1/L_f$, $c_3 = \Theta(1)$ in \Cref{algo2}, and for the deterministic case $\alpha = 1/L_f$, $c_3 = \Theta(\kappa)$ in \Cref{algo2}, we obtain a sample complexity of $O(\epsilon^{-2} \kappa^{16})$ and $O(\epsilon^{-1}\log(\epsilon^{-1})\kappa^{5})$ respectively for the stochastic case of \Cref{th:finalconv}\ref{resone} and the deterministic case of \Cref{th:finalconvdet} where $\kappa = (1-q)^{-1}$. For LL problems of type \eqref{minminstochprob} with Lipschitz smooth and strongly convex loss, by appropriately setting $\eta$ in \eqref{eq:gradmap}, $\kappa$ is proportional to the condition number of the LL problem. In comparison, Amigo \Citep{arbel2021amortized} reaches a sample complexity of $O(\epsilon^{-2} \kappa^{9})$ (stochastic) and $O(\epsilon^{-1} \kappa^{4})$ (deterministic).
However, we note that for the deterministic case by setting $t_s=k_s=\Theta(\kappa\log(\kappa s))$ we obtain a sample complexity of $O(\epsilon^{-1}\kappa^{4}\log(\kappa \epsilon^{-1}))$, which is worse than warm-start only for the log factor. %
Finally, we note that \citep{arbel2021amortized} have a stronger assumption, which in our setting can be formulated as
$
  \norm{\jac_2 \Phi(w, \lambda)} \leq \boundjac \quad \forall w \in \R^d, \lambda \in \Lambda.
$
If we make such assumption (which implies Assumption~\ref{ass:add}\ref{ass:boundjac}), use different stepsizes in the LL and LS, i.e.\@ $\eta_t = \beta/(\gamma + t)$ with $\beta = 2/(1-q^2)$, $\gamma = \beta(1+\sigma_2')$ (LL) or $\gamma = \beta(1+\tilde{\sigma}_2)$ (LS) in \Cref{algo1} and set $t_s=k_s =J_s = \Theta(\kappa^3 S)$ (i.e.\@ $c_3 = \Theta(\kappa^3)$) we also obtain a stochastic complexity of $O(\epsilon^{-2} \kappa^{9})$. %

\end{remark}

 \begin{remark}[An advantage of warm-start]\label{rm:advws}
Our sample complexity results as well as those in \cite{ghadimi2018approximation} depend on the constant $B$, defined in Assumption~\ref{ass:add}\ref{ass:boundw} such that $\norm{w_0(\lambda) - w(\lambda)} \leq B \ \forall \lambda \in \Lambda$. Instead, warm-start complexity bounds do not require such assumption and instead depend only on the quantity $\norm{w_0(\lambda_0) - w(\lambda_0)}$, which can be much smaller than $B$; see e.g. \citep{arbel2021amortized}. Although our method matches the sample complexity of warm-start approaches in the parameter $\epsilon$, this aspect may lead to better bounds for warm-start, thus explaining why it is generally advantageous in practice.
\end{remark}

\section{Experiments}\label{se:experiments}

We design the experiments with the following goals. Firstly, we assess the difficulties of applying warm-start and the effect of different upper-level batch sizes in a classification problem involving equilibrium models and in a meta-learning problem. In both settings the lower-level problem can be divided into several smaller sub-problems. Secondly, we compare our method with others achieving near-optimal sample complexity in a data poisoning experiment. All methods have been implemented in PyTorch \citep{pytorch} and the experiments have been executed on a GTX 1080 Ti GPU with $11$GB of dedicated memory. 

\begin{figure*}[tp!]
    \newcommand{\deqwidth}{0.492}
    \centering
    \begin{subfigure}
        \centering
        \includegraphics[width=\deqwidth\textwidth]{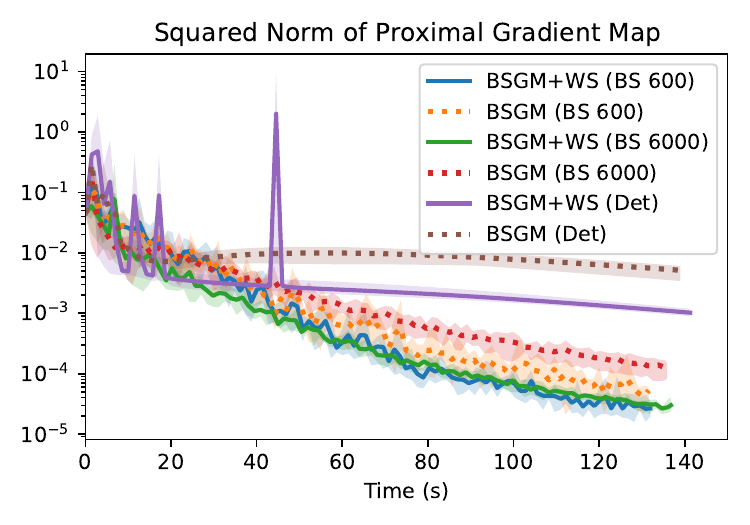}
    \end{subfigure}
    \begin{subfigure}
        \centering
        \includegraphics[width=\deqwidth\textwidth]{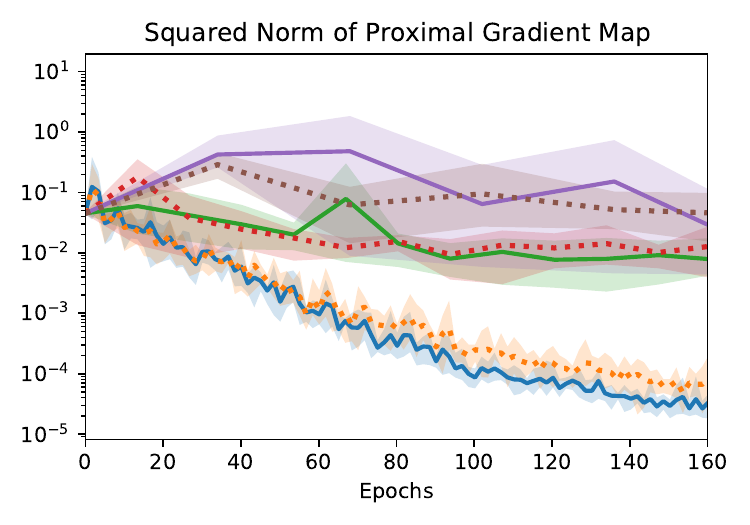}
    \end{subfigure}
    \begin{subfigure}
        \centering
        \includegraphics[width=\deqwidth\textwidth]{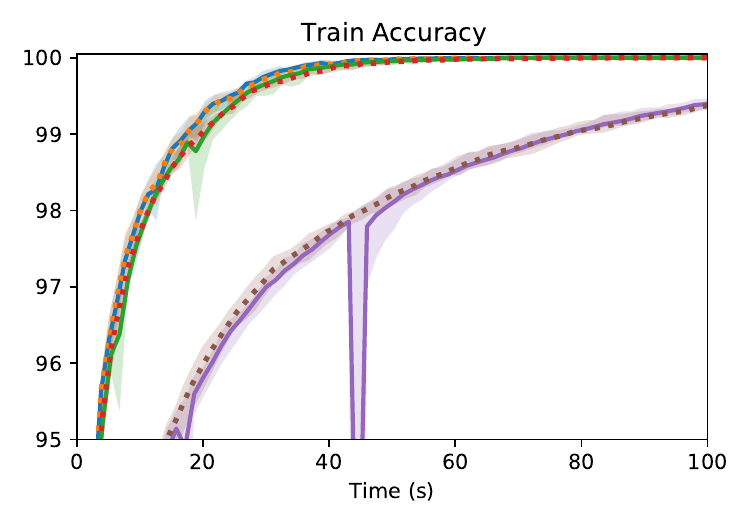}
    \end{subfigure}
        \begin{subfigure}
        \centering
        \includegraphics[width=\deqwidth\textwidth]{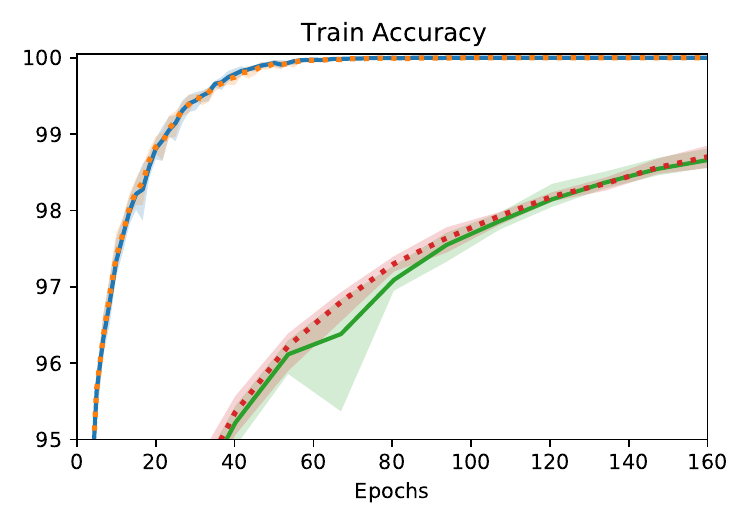}
    \end{subfigure}
    \begin{subfigure}
        \centering
        \includegraphics[width=\deqwidth\textwidth]{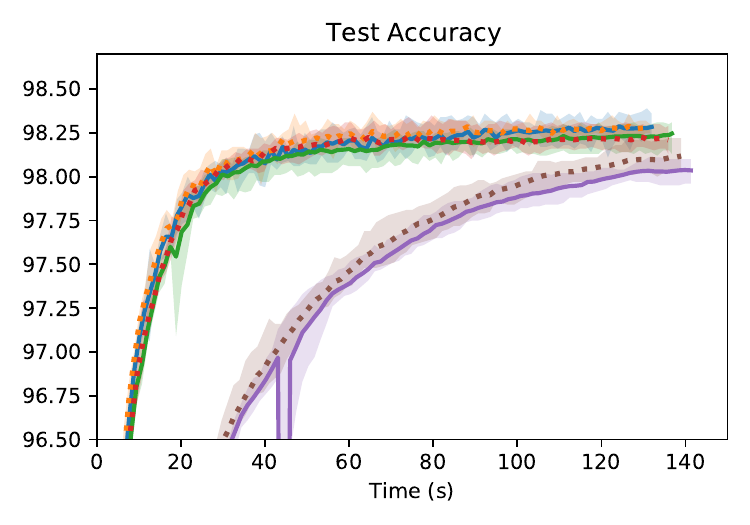}
    \end{subfigure}
    \begin{subfigure}
        \centering
        \includegraphics[width=\deqwidth\textwidth]{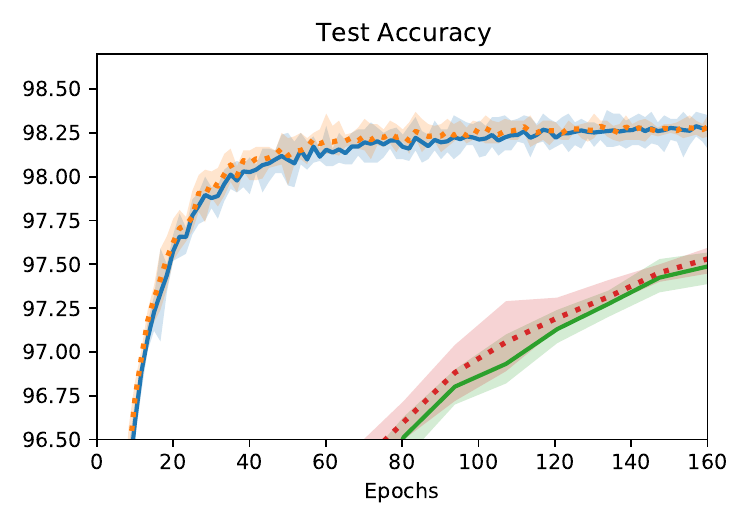}
    \end{subfigure}
    \caption{\textit{Equilibrium Models on MNIST.} Results show mean (solid, dashed and dotted lines) and max-min (shaded region) over 5 seeds varying the randomness in the mini-batches and the initialization. BSGM is the method in \Cref{algo2} while BSGM+WS is the variant with warm-start on the LL. BS indicates the mini-batch size used while methods with Det in the name use the whole training set of $60$K examples. 
    }\label{fig:deq}
\end{figure*}

\subsection{Equilibrium Models}

We consider a variation of the equilibrium models experiment presented in \cite[Section~3.2]{grazzi2020iteration}. In particular, we consider a multi-class classification problem with the following bilevel formulation:
\begin{equation}
\begin{gathered}
    \min_{\lambda \in \Lambda} \sum_{i=1}^{n} \CE( \theta w(\lambda)^i + b, y_i)\\
    \text{\ subject~to ~} w(\lambda)^i = \tanh(A w(\lambda)^i + B X_i + a) \quad \forall i \in \{1,\dots, n\}
\end{gathered}
\end{equation}
where $\CE$ is the cross-entropy loss, $(X, y) \in \R^{n\times p} \times \{1, \dots, c\}^n$ is the training set, $\lambda = (\theta, b, A, B, a)$, $\Lambda = \{\theta \in \R^{c\times d} \,:\, \norm{\theta}_{\infty} \leq 1 \} \times \R^c \times \{A \in \R^{d\times d} \,:\, \norm{A} \leq 0.5 \} \times \R^{d\times p} \times \R^d$ and $w(\lambda)^i \in \R^{d}$ is the fixed point representation for $i$-th training example. The constraint on $A$, guarantees that for all $i$, the map $w \mapsto \tanh(A w + Bx_i +c)$ is a contraction with Lipschitz constant not greater than $0.5$. We perform this experiments using the whole MNIST training set, hence $n=6 \times 10^4, p = 784, c = 10$, and set $d=200$.

We compare variants of BSGM (\Cref{algo2}) with different batch sizes ($J_s$ in \Cref{algo2}), which in this case indicates the number of training examples used to estimate the gradients of the UL objective. Moreover, we evaluate an extension of BSGM which uses warm-start only on the LL problem (similar to StochBiO \citep{ji2021bilevel}).  Note that when using warm-start, all the fixed point representations computed by the algorithm are stored in memory to be used in the future. When the ratio between the number of examples $n$ and the batch size is large, this can greatly increase the memory cost of the algorithm compared to the procedure without warm-start. 
For this particular problem, this cost is manageable since it amounts to storing a total of $n d = 12 \times 10^6$ floats, which correspond to $48$ MB of memory, but for higher values of $d$ and $n$ it quickly becomes prohibitive, as we show in the meta-learning experiment.

Let $\lambda_0 = (\theta_0, b_0, A_0, B_0, a_0)$ be the hyperparameters at initialization, we set $b_0 = 0$, and we sample each coordinate of $\theta_0, A_0, B_0,$ and $a_0$ from a Gaussian distribution with zero mean and standard deviation $0.01$. In \Cref{algo2} we also set $w_0 (\lambda) = 0$, $t_s=k_s = 2$, and $\alpha = 0.5$. Since computing the map $w \mapsto \tanh(A w + Bx_i +a)$ is relatively cheap, we use deterministic solvers with step-size $1$ for the LL and LS of each training example. To evaluate the UL parameters found by the algorithms, we compute an accurate approximation of the LL solution and the hypergradient on all training examples by running the LL and LS solver for $20$ steps. The proximal gradient map is computed according to \eqref{eq:proxgradmap} with $\alpha=1$.

Results are shown in \Cref{fig:deq}, where we compare three key performance measures of the different methods versus time and  number of epochs. When comparing methods using the same batch size we can see that using warm-start improves the performance in terms of the norm of the proximal gradient map, i.e.~the quantity that we can control theoretically. However, this effect decreases with smaller batch sizes since more UL iterations can pass until the same example is sampled twice. Furthermore,  train and test accuracy are similar for methods with the same batch size, regardless of the use of warm-start. Finally, we note that decreasing the mini-batch consistently improves the performance in terms of number of epochs while, thanks to the parallelism of the GPU, the performance with batch size equal to $600$ and $6000$ are similar.

\subsection{Meta-Learning}\label{se:metalearningexp}
We perform a meta-learning experiment on Mini-Imagenet \citep{vinyals2016matching}, a popular few-shot classification benchmark. Mini-Imagenet contains 100 classes from Imagenet which are split into 64, 16, 20 for the meta-train, meta-validation and meta-test sets respectively. A task is constructed by selecting some images from  $c$ randomly selected classes.
Each image is downsampled to $84 \times 84$ pixels.
Similarly to \cite{franceschi2018bilevel}, we evaluate an hyper-representation model  where the UL parameters are the parameters of the representation layers of a convolutional neural network (CNN), shared across tasks, while the task-specific LL parameters are the parameters of the last linear layer. The CNN is composed by stacking 4 blocks, each made by a $3 \times 3$ convolutions with $32$ output channels followed by a batch normalization layer. 

We evaluate the performance of \Cref{algo2} where the network parameters $\lambda_0$ are initialized using the default random initialization in PyTorch, $w_0(\lambda) = 0$, $\alpha = 0.2$,  $\eta_j = 0.05$,  $t_s= 10$, and different batch sizes $J_s = \{8, 16, 32\}$. The batch size in this case corresponds to the number of tasks at each UL iteration. Using warm start in this setting could require to save the last linear layer for all tasks, hence $n \times d \times c$ floats, where $n$ is the number of tasks and $d \times c$ are the number of weights in the last linear layer. A meta-training task is constructed by selecting $c=5$ classes out of $64$, hence the number of tasks is $n=7,\!624,\!512$. Moreover, we set $d=800$. Thus, storing the last layer for all tasks would require $122$ GB of storage, which largely exceeds our GPU memory.
Furthermore, the ratio between $n$ and batch size is very high and this is likely to make the effect of using warm-start negligible.

Results are shown in \Cref{fig:meta}, where we see that methods with smaller batch-sizes converge faster despite requiring a higher number of UL iterations. Furthermore, since during meta-training we see only $50,000$ tasks, we also implemented the method using warm-start by storing the approximate solutions to all previously sampled tasks to be used as initialization when they are sampled again. We run the method with mini-batch size equal to 8 and for 5 seeds and observed that all metrics essentially overlap the ones without warm-start, while the memory cost increases by $0.8$ GB. These experiments suggest that warm-start may be ineffective in meta-learning problems, as mentioned in the introduction. Indeed, in this setting we observed that each task is sampled at most $3$ times in a total of $6,250$ iterations.

\begin{figure*}[t]
    \newcommand{\metawidth}{0.492}
    \centering
    \begin{subfigure}
        \centering
        \includegraphics[width=\metawidth\textwidth]{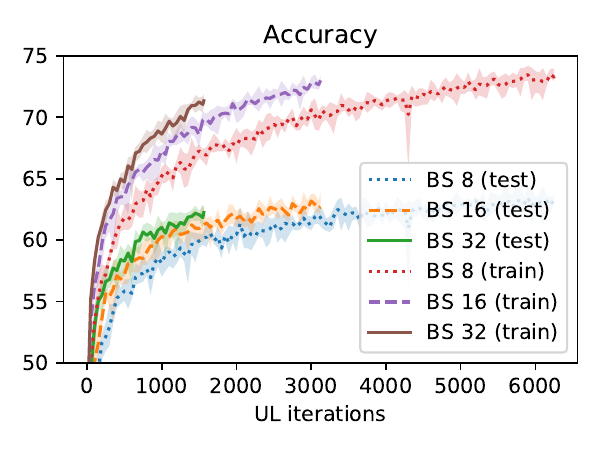}
    \end{subfigure}
    \begin{subfigure}
        \centering
        \includegraphics[width=\metawidth\textwidth]{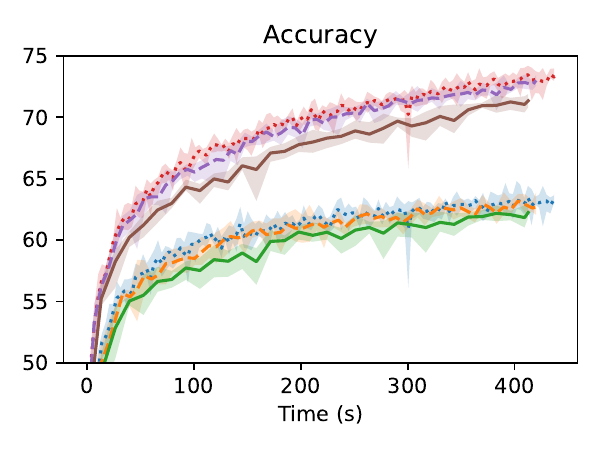}
    \end{subfigure}
    \vspace{-.8cm}
    \caption{\textit{5-way 5-shot classification on Mini-Imagenet.}
    The plot show mean (solid lines) and $\max-\min$ (shaded region) over 5 runs. Values are the average accuracy over $1000$  meta-train/meta-test tasks computed after $10$ steps of the LL solver. At the end of training all methods have seen a total of $50$K tasks.
    }\label{fig:meta}
\end{figure*}

\subsection{Data Poisoning}

We consider the \textit{data poisoning} scenario where a malicious agent or \textit{attacker} aims at decreasing the performance of a machine learning model by corrupting its training data set. In particular, the attacker adds noise to some training examples. However, this noise must be small in magnitude to avoid for the attack to be uncovered.

Specifically, we consider an image classification problem on the MNIST data set where $(X, y) \in \R^{n \times p} \times \{1,\dots,c\}^n$,  and $(X', y' ) \in \R^{n' \times p} \times \{1,\dots,c\}^n$ are the training and validation sets, and $p = 784$, $c = 10$, $n = 45,\! 000$ and $n' = 15,\! 000$ are the number of features, classes, training examples and validation examples respectively. Furthermore, we randomly select $\Ic \subseteq \{1,\dots, n\}$ to be the indices of the corrupted training examples such that $|\Ic| = 9,\!000$. The attacker finds the noise $\lambda$ by solving the following bilevel optimization problem.
\begin{equation}
\begin{gathered}
    \max_{\lambda \in \Lambda} \frac{1}{n'}\sum_{i=i}^{n'}\CE( w(\lambda)^\top X'_i, y'_i)\\
    \text{\ subject~to ~}w(\lambda) = \argmin_{w \in \R^{p\times c}} \frac{1}{n}\sum_{i=1}^{n}\CE(w^\top (X_i + \lambda_i) , y_i) +\frac{0.1}{p} \norm{w}^2,
\end{gathered}
\end{equation}
where $\CE$ is the cross-entropy loss, $\Lambda = \{\lambda \in \R^{n \times p} \,|\, \lambda_i \in  \ball_2 (0, 5) \ \forall i \in \Ic, \lambda_i = 0 \ \forall i \in \{1,\dots, n\}/\Ic \}$ and $\ball_2 (0, 5)$ is the $p$-dimensional L$2$-ball centered in $0$ with radius $5$. 
Note that the LL problem is both strongly convex and Lipschitz smooth.

\textit{Baselines.} We compare our method with  StochBiO \citep{ji2021bilevel}, Amigo \citep{arbel2021amortized}, ALSET \citep{chen2021single}, which achieve (near) optimal sample complexity. We also consider ALSET$^\dag$, i.e. a variant of ALSET  where the LS problem is solved using warm-start and only one iteration. All baselines have been implemented as extensions to \Cref{algo2} specialized to LL problems of type~\eqref{minminstochprob}, which differ only in the use of warm-start and in the number of iterations and batch-sizes used.
Except for  ALSET$^\dag$-DET, which is the deterministic version of ALSET$^\dag$  and computes the LL objective exactly, all other methods  use mini-batches of size $90$ to estimate the LL objective and its derivatives. We found this value to be sufficiently large for Amigo and StochBiO to perform well. The UL objective is instead always computed using all $15$K validation examples.
To fairly evaluate  the different bilevel optimization methods, the linear model used for the final evaluation is trained by $1000$ steps of gradient descent on the LL objective
\begin{equation*}
   \frac{1}{n}\sum_{i=1}^{n}\CE(w^\top(X_i + \lambda^*), y_i) + \frac{0.1}{p} \norm{w}^2,
\end{equation*}
where $\lambda^*$ is the output of the bilevel optimization method.

\textit{Random Search.} Bilevel optimization methods have several configuration parameters which greatly affect the performance, e.g.~the number of iterations for the LL and LS solvers, step sizes for the UL, LL and LS. Theoretical values for these parameters are often too conservative, hence they are usually set via manual search which is hard to reproduce and may be suboptimal. Thus, for a better comparison, we set a total budget of $2$M single-sample gradients and hessian-vector products, so that each algorithm uses the same number of samples\footnote{We do not account for the difference in computational cost between gradients and hessian vector-products. The latter are usually more costly in practice.}, and perform a random search with 200 random configuration parameters to select the configurations achieving the lowest accuracy on the validation set. Values and ranges of the random search are shown in \Cref{tab:randomsearch}.
Note that to reduce the number of configuration parameters we keep them unchanged across UL and LL/LS iterations.
For our method, we observed that using fixed instead of decreasing stepsizes for the LL/LS does not affect the top performances after the random search.
Furthermore, we set $k=t$ and $\eta_{\text{LL}} = \eta_{\text{LS}}$ only for our method and all the others which use warm-start both for the LL and LS problems, which we observed that improves the performance\footnote{Indeed, we observed that using $k\neq t$ and $\eta_{\text{LL}} \neq \eta_{\text{LS}}$ for BSGM and Amigo does not improve and usually decreases the performance of the best methods, while setting $k= t$ and $\eta_{\text{LL}} = \eta_{\text{LS}}$ decreases the performance of StochBiO.}.

\textit{Results.} In \Cref{tab:poisoning} we show the results. Our method (BSGM) outperforms all the single-loop bilevel optimization methods (ALSET$^\dag$ and ALSET). However, methods using warm-start only in the LL (StochBiO) and both in LL and LS (Amigo) outperform BSGM, albeit not by a large margin.  To aid reproducibility, we report in \Cref{tab:besthparams} the best  configuration parameters of each method.

\begin{table}[H]
\centering
\begin{tabular}{@{}llllllll@{}}
\toprule
\textbf{Method}        & WS  & $t$         & $k$         & $J$             & $\alpha$         & $\eta_{\text{LL}}$          & $\eta_{\text{LS}}$                \\ \midrule
StochBiO               & Y,N & $[10:10^4]$ & $[10:10^4]$ & $k$             & $[10^3:10^9]$  & $[10^{-4}:10]$ & $[10^{-4}:10]$      \\
Amigo                  & Y,Y & $[10:10^4]$ & $t$         & $t$             & $[10^3:10^9]$  & $[10^{-4}:10]$ & $\eta_{\text{LL}}$               \\
\textbf{BSGM (ours)}   & N,N & $[10:10^4]$ & $t$         & $t$             & $[10^3:10^9]$  & $[10^{-4}:10]$ & $\eta_{\text{LL}}$              \\
ALSET$^\dag$-DET            & Y,Y & $1$         & $1$         & $1$             & $[10^3:10^9]$  & $[10^{-4}:10]$ & $\eta_{\text{LL}}$              \\
ALSET$^\dag$                & Y,Y & 1           & 1           & 1               & $[10^3:10^9]$  & $[10^{-4}:10]$ & $\eta_{\text{LL}}$              \\
ALSET                  & Y,N & 1           & $[10:10^4]$ & 1               & $[10^3:10^9]$  & $[10^{-4}:10]$ & $[10^{-4}:10]$      \\ \bottomrule
\end{tabular}
    \caption{\textit{Configurations parameters for the random search}. The WS column indicates whether warm-start is used (Y) or not (N) for the LL (first entry) and LS (second entry). $t$, $k$ and $J$ are respectively the number of iteration for the LL and LS and the batch size, while $\alpha$, $\eta_{\text{LL}}$, and $\eta_{\text{LS}}$ are the step sizes for the UL, LL and LS respectively.  Configuration parameters are sampled according to the log-uniform distribution over the specified ranges. For all methods we set $\lambda_0 = 0$.}
    \label{tab:randomsearch}
\end{table}

\begin{table}[H]
    \centering
\begin{tabular}{lrrr}
\toprule
     \textbf{Method} &  \textbf{Test (Val) Best} & \textbf{Test (Top 10)} &  \textbf{Val (Top 10)} \\
\midrule
   StochBiO                 &     76.78 (73.57) & 79.97 $\pm$ 1.92 &         77.33 $\pm$ 2.28 \\
      Amigo                 &     78.01 (75.09) & 79.29 $\pm$ 0.94 &         76.27 $\pm$ 0.93 \\
       \textbf{BSGM (ours)} &     78.05 (75.05) & 80.90 $\pm$ 1.33 &         78.16 $\pm$ 1.48 \\
ALSET$^\dag$-DET                 &     83.03 (80.30) & 86.13 $\pm$ 1.38 &         84.10 $\pm$ 1.73 \\
       ALSET$^\dag$              &     90.75 (89.99) & 90.66 $\pm$ 0.13 &         90.19 $\pm$ 0.15 \\
      ALSET                 &     90.89 (90.49) & 90.99 $\pm$ 0.11 &         90.65 $\pm$ 0.10 \\
\bottomrule
    \end{tabular}
    \caption{\textit{Data-poisoning Accuracy (Lower is better)}. We report values for best and top 10 best performing parameter configurations selected via random search. For the top 10 results we report mean $\pm$ standard deviation. ALSET$^\dag$-DET is the best performing deterministic method, all the others are stochastic.}
    \label{tab:poisoning}
\end{table}

\begin{table}[H]
    \centering

\begin{tabular}{llllllll}
\toprule
\textbf{Method} & Test (Val) Acc & $t$       & $k$       & $J$       & $\alpha$         & $\eta_{\text{LL}}$          & $\eta_{\text{LS}}$  \\ \midrule
\midrule
   StochBiO &  76.78 (73.57) & 418 & 2477 & $k$ & $1.0\times10^{6}$ & $5.4\times10^{-3}$ & $1.3\times10^{-2}$ \\
      Amigo &  78.01 (75.09) & 155 &    $t$ & $t$ & $1.0\times10^{7}$ & $1.1\times10^{-2}$ &     LL sz \\
\textbf{BSGM (ours)} &  78.05 (75.05) & 287 &    $t$ & $t$ & $4.0\times10^{8}$ & $9.0\times10^{-2}$ &     LL sz \\
   ALSET$^\dag$-DET &  83.03 (80.30) &   1 &    1 & 1 & $1.8\times10^{5}$ & $5.6\times10^{-1}$ &     LL sz \\
       ALSET$^\dag$ &  90.75 (89.99) &   1 &    1 & 1 & $1.6\times10^{6}$ & $5.3\times10^{-2}$ & $3.9\times10^{-1}$ \\
      ALSET &  90.89 (90.49) &   1 &   85 & 1 & $5.5\times10^{8}$ & $2.0\times10^{-2}$ & $2.7\times10^{-1}$ \\
\bottomrule
\end{tabular}
    \caption{\textit{Best configuration parameters}. Configuration parameters with lowest validation accuracy  among 200 random configurations for each method.}
    \label{tab:besthparams}
\end{table}

\section{Conclusions}

In this paper, we studied bilevel optimization problems where the upper-level objective is smooth and the lower-level solution is the fixed point of a smooth contraction mapping.  
In particular, we presented BSGM (\Cref{algo2}), a bilevel optimization procedure based on inexact gradient descent, where the inexact gradient is computed via \SID{} (\Cref{algo1}). \SID{} uses stochastic fixed-point iterations to solve both the lower-level problem and the linear system and estimates $\nabla\fo$ and $\jac_2\Phi$ using large mini-batches. 
We proved that, even without the use of warm-start on the lower-level problem and the linear system, BSGM achieves optimal and near-optimal sample complexity in the stochastic and deterministic bilevel setting respectively. We stress that in recent literature, warm-start was thought  to be crucial to achieve the optimal sample complexity.
We also showed that, when compared to methods using warm-start, our approach 
yields a simplified and modular analysis which does not deal with the interactions between upper-level and lower-level iterates. Moreover, we showed empirically the inconvenience of the warm-start strategy on equilibrium models and meta-learning. Finally, we compared our method with several bilevel methods relying on warm-start on a data-poisoning experiment.

\acks{This work was supported in part by the EU Projects ELISE and ELSA, as well the PNNR Project FAIR.
We thank all anonymous reviewers for their useful insights and suggestions.}

\raggedbottom

\appendix

\section{Main Proofs}

\begingroup
\renewcommand{\wstoc}[1][t]{w_{#1}(\lambda)}
\renewcommand{\vstocj}[2][k]{ v_{{#1}}({#2},\lambda)}

\subsection{Proof of Lemma~\ref{lm:lemmanew}}\label{pr:lemmanew}

To prove \ref{wprimelm}, recall that $w'(\lambda) =  \big(I -\jac_1 \Phi(w(\lambda),\lambda)\big)^{-1}\jac_2 \Phi(w(\lambda), \lambda)$, hence
\begin{align*}
    \norm{w'(\lambda)} &= \norm{\big(I -\jac_1 \Phi(w(\lambda),\lambda)\big)^{-1}\jac_2 \Phi(w(\lambda), \lambda)} \\
    &\leq \norm{\big(I -\jac_1 \Phi(w(\lambda),\lambda)\big)^{-1}}\norm{\jac_2 \Phi(w(\lambda), \lambda)}
     \\&\leq \sum_{i=0}^{\infty} \norm{\jac_1 \Phi(w(\lambda),\lambda)}^i\norm{\jac_2 \Phi(w(\lambda), \lambda)} 
    \leq \sum_{i=0}^{\infty} \q^i \boundjac =  \frac{\boundjac}{1-\q},
\end{align*}
where in the second inequality we used the properties of Neumann series and in the last inequality we used Assumption~\ref{ass:hypergrad}\ref{ass:contraction} and \ref{ass:add}\ref{ass:boundjac}.

Next we prove \ref{wprimelip}. Let $A(\lambda) = I -\jac_1 \Phi(w(\lambda),\lambda)$ For every $\lambda \in \Lambda$
\begin{align*}
    \norm{A(\lambda_1) - A(\lambda_2)} &= \norm{\jac_1 \Phi(w(\lambda_1),\lambda_1)- \jac_1 \Phi(w(\lambda_2),\lambda_2)} \\
    &\leq \norm{\jac_1 \Phi(w(\lambda_2),\lambda_1)- \jac_1 \Phi(w(\lambda_2),\lambda_2)} \\
    &\quad + \norm{\jac_1 \Phi(w(\lambda_1),\lambda_1)- \jac_1 \Phi(w(\lambda_2),\lambda_1)} \\
    &\leq \lipw\norm{\lambda_1 - \lambda_2} + \rhow\norm{w(\lambda_1) - w(\lambda_2)} \\
    &\leq \Big( \lipw + \frac{\rhow\boundjac}{1-\q} \Big)\norm{\lambda_1 - \lambda_2},
\end{align*}
where we used Assumption~\ref{ass:hypergrad}\ref{ass:rho} and \ref{ass:add}\ref{ass:lipwl} in the second inequality and \ref{wprimelm} in the last inequality.
Consequently, for every $\lambda_1,\lambda_2 \in \Lambda$
\begin{align*}
    \norm{w'(\lambda_1) - w'(\lambda_2)} &\leq \norm{A(\lambda_1)^{-1}}\norm{ \jac_2 \Phi((w(\lambda_1), \lambda_1) - \jac_2 \Phi((w(\lambda_2), \lambda_2)} \\
    &\quad + \norm{\jac_2 \Phi((w(\lambda_1), \lambda_1)}\norm{A(\lambda_1)^{-1}}\norm{A(\lambda_1) - A(\lambda_2)}\norm{A(\lambda_2)^{-1}} \\
    &\leq \norm{A(\lambda_1)^{-1}}\norm{ \jac_2 \Phi((w(\lambda_1), \lambda_2) - \jac_2 \Phi((w(\lambda_2), \lambda_2)} \\
    &\quad + \norm{A(\lambda_1)^{-1}}\norm{ \jac_2 \Phi((w(\lambda_1), \lambda_1) - \jac_2 \Phi((w(\lambda_1), \lambda_2)} \\
    &\quad + \norm{\jac_2 \Phi((w(\lambda_1), \lambda_1)}\norm{A(\lambda_1)^{-1}}\norm{A(\lambda_1) - A(\lambda_2)}\norm{A(\lambda_2)^{-1}} \\ 
    &\leq \Bigg[\frac{\rhol\boundjac/(1-\q) +\lipl}{1-\q} + \frac{\boundjac}{(1-\q)^2} \Big( \lipw + \frac{\rhow\boundjac}{1-\q} \Big)\Bigg]\norm{\lambda_1 - \lambda_2}.
\end{align*}

To prove \ref{lipgradient} instead, let 
\begin{equation}
    \bar\nabla f(w,\lambda) := \nabla_2 \fo (w, \lambda) + \jac_2 \Phi (w,\lambda) \big[I - \jac_1 \Phi(w,\lambda)^\top\big]^{-1}\nabla_1 \fo (w,\lambda)
\end{equation}
Note that $\nabla f(\lambda) = \bar\nabla f(w(\lambda),\lambda)$. We have that for every $\lambda_1,\lambda_2 \in \Lambda$
\begin{equation}\label{eq:boundgrad}
    \norm{\nabla f(\lambda_1) - \nabla f(\lambda_2)} \leq \norm{\nabla f(\lambda_1) - \bar\nabla f(w(\lambda_1), \lambda_2)} + \norm{\nabla f(\lambda_2) - \bar\nabla f(w(\lambda_1), \lambda_2)} 
\end{equation}
We bound the two terms of the RHS of \eqref{eq:boundgrad} as follows.
\begin{align*}
    \norm{\nabla f(\lambda_1) - \bar\nabla f(w(\lambda_1), \lambda_2)} &\leq \norm{\nabla_2 \fo (w(\lambda_1), \lambda_1) - \nabla_2 \fo (w(\lambda_1), \lambda_2))} + \\
    &\quad + \norm{w'(\lambda_1)}\norm{\nabla_1 \fo (w(\lambda_1), \lambda_1) - \nabla_1 \fo (w(\lambda_1), \lambda_2))} \\
    &\leq \big(\Lolb + \frac{\boundjac\Lola}{1-\q}\big)\norm{\lambda_1 -\lambda_2},
\end{align*}
\begin{align*}
\norm{\nabla f(\lambda_2) - \bar\nabla f(w(\lambda_1), \lambda_2)} &\leq  \norm{\nabla_2 \fo (w(\lambda_2), \lambda_2) - \nabla_2 \fo (w(\lambda_1), \lambda_2))}    \\
&\quad + \norm{w'(\lambda_2)}\norm{\nabla_1 \fo (w(\lambda_2), \lambda_2) - \nabla_1 \fo (w(\lambda_1), \lambda_2))} \\
&\quad + \norm{\nabla_1 \fo(w(\lambda_1), \lambda_2)}\norm{w'(\lambda_2) - w'(\lambda_1)} \\ &\leq \Big(\Bo\lipwprime + \frac{\Lol\boundjac}{1-\q} + \frac{\Low\boundjac^2}{(1-\q)^2}\Big)\norm{\lambda_1 - \lambda_2}.
\end{align*}
Summing the two inequalities above we obtain the final result.

\subsection{Proof of \Cref{th:boundbias}}\label{pf:bias}

\begin{proof}
\ref{th:boundbias_i}:
Using the definition of $\hat \nabla f(\lambda)$ and the fact that $\zeta'_j$ and $\vstocj{\wstoc}$
are independent random variables, we get
\begin{equation*}
    \Exp{\hat \grad f(\lambda) \given \wstoc} = \grad_2 \fo(\wstoc, \lambda) + \jac_2 \Phi(\wstoc, \lambda)^\top \Exp{\vstocj{\wstoc}\given \wstoc}.
\end{equation*}
Consequently, recalling the hypergradient equation, we have,
\begin{align}
\nonumber &\norm[\big]{ \Exp{\hat \grad f(\lambda) \given \wstoc}   - \grad f(\lambda)} 
\nonumber \\
\nonumber&\qquad\qquad \leq \norm{\grad_2 \fo(w(\lambda), \lambda) - \grad_2 \fo(\wstoc, \lambda)} 
\\\nonumber&\qquad\qquad\qquad + \big\lVert\jac_2 \Phi(w(\lambda), \lambda)^\top \vopt{w(\lambda)} 
- \jac_2 \Phi(\wstoc, \lambda)^\top \Exp{\vstocj{\wstoc}\given \wstoc} \big\rVert\\
\nonumber&\qquad\qquad \leq \norm{\grad_2 \fo(w(\lambda), \lambda) - \grad_2 \fo(\wstoc, \lambda)} \\
\nonumber&\qquad\qquad\qquad+ 
\norm{\jac_2 \Phi(w(\lambda), \lambda)}\norm{\vopt{w(\lambda)}
     - \Exp{\vstocj{\wstoc}\given \wstoc}} \\
\label{eq:20200601d}    &\qquad\qquad\qquad+ \norm{\jac_2 \Phi(w(\lambda), \lambda) - \jac_2 \Phi(\wstoc, \lambda)} \norm{\Exp{\vstocj{\wstoc}\given \wstoc}}.
\end{align}
Now, concerning the term $\norm{\vopt{w(\lambda)} - \Exp{\vstocj{\wstoc} \given \wstoc}}$ in the above inequality,
we have
\begin{align}
\nonumber\lVert \vopt{w(\lambda)} - &\Exp{\vstocj{\wstoc} \given \wstoc}\rVert \\[1ex]
\label{eq:20200601b}&\leq \norm{\vopt{w(\lambda)} - \vopt{\wstoc}} + \norm{\vopt{\wstoc} - \Exp{\vstocj{\wstoc}\given \wstoc}}.
\end{align}
Since $\Exp{\voptj{\wstoc} \given \wstoc} = \vopt{\wstoc}$ we have
\begin{align*}
    &\norm{\vopt{\wstoc} - \Exp{ \vstocj{\wstoc} \given \wstoc } } 
    = \norm{\Exp{\voptj{\wstoc} - \vstocj{\wstoc} \given \wstoc}}
\end{align*}
Moreover, using  Jensen inequality and \Cref{ass:innerbackrates} we obtain
\begin{align}
\nonumber \norm{\Exp{\voptj{\wstoc} - \vstocj{\wstoc} \given \wstoc}} &=
\sqrt{ \norm{\Exp{\voptj{\wstoc}  - \vstocj{\wstoc}\given \wstoc}}^2 }\\[1ex]
\nonumber    & \leq \sqrt{\Exp{\norm{\voptj{\wstoc}  - \vstocj{\wstoc}}^2 \given \wstoc}}\\[1ex]
\label{eq:20200601c}    & \leq \sqrt{\hrf(k)}.
\end{align}
Therefore, using \Cref{lm:lipv},  \eqref{eq:20200601b} yields
\begin{equation}
\label{eq:20200601e}
\norm{\vopt{w(\lambda)} - \Exp{\vstocj{\wstoc} \given \wstoc}} \leq 
\left(\frac{\rhow\Bo}{(1-\q)^2} + \frac{\Low}{1-\q} \right)\norm{w(\lambda) - \wstoc} + \sqrt{\hrf(k)}.
\end{equation}
In addition, it follows from \eqref{eq:20200601b}-\eqref{eq:20200601c} and \cref{lm:normv} that
\begin{align}
\label{eq:20200601f}
\nonumber     \norm{\Exp{\vstocj{\wstoc}\given \wstoc}} &\leq \norm{\vopt{\wstoc}} 
     + \norm{\vopt{\wstoc}- \Exp{\vstocj{\wstoc}\given \wstoc}} \\
     &\leq \frac{\Bo}{1-\q} + \sqrt{\hrf(k)}.
\end{align}
Finally, combining \eqref{eq:20200601d}, \eqref{eq:20200601e}, and \eqref{eq:20200601f}, and
using \Cref{ass:hypergrad}, \ref{th:boundbias_i} follows.
Then, since
\begin{equation*}
    \norm{\EE[\hat\grad f(\lambda)] - \grad f(\lambda)} = \norm[\big]{ \Exp[\big]{ \Exp{\hat \grad f(\lambda)\given \wstoc}   - \grad f(\lambda)}}
   \leq  \Exp[\big]{ \norm[\big]{ \Exp{\hat \grad f(\lambda)\given \wstoc}   - \grad f(\lambda)}},
\end{equation*}
\ref{th:boundbias_ii} follows by taking the expectation in \ref{th:boundbias_i}, using \Cref{ass:innerbackrates} and that 
$\Exp{\hat \Delta_w} = \sqrt{(\Exp{\hat \Delta_w})^2} 
\leq \sqrt{\Exp{\hat \Delta^2_w}} \leq \sqrt{\rf(t)}$.

\end{proof}

\subsection{Proof of \Cref{th:varboundone}}\label{pf:varboundone}

\begin{proof}
Let $\Exptilde{\cdot} := \Exp{\,\cdot \given \wstoc}$,  $\Vartilde{\cdot} : = \Var{\,\cdot \given \wstoc}$, $b_1 := \jac_2 \Phi(\wstoc, \lambda)^\top\vstocj{\wstoc}$ and $b_2: = \Vartilde{\nabla_2  \bar \fo_{\J}(\wstoc{}, \lambda)}$. 
Then,
\begin{align*}
\Vartilde{\hat \grad f(\lambda)} &= \Exptilde[\big]{\norm{\hat \grad f(\lambda) - \Exptilde{\hat \grad f(\lambda)}}^2} \\
&\leq 2\Exptilde[\big]{\norm{\jac_2 \Phi(\wstoc, \lambda)^\top\Exptilde{\vstocj{\wstoc}} \mp b_1 - \jac \bar\Phi_{\J}(\lambda)^\top \vstocj{\wstoc}}^2} + 2b_2 \\
&\leq 2\norm{\jac_2 \Phi(\wstoc, \lambda)}^2 \Exptilde[\big]{\norm{\vstocj{\wstoc} - \Exptilde{\vstocj{\wstoc}}}^2} \\
&\quad+ 2\Exptilde[\big]{\norm{\vstocj{\wstoc}}^2}\Exptilde[\big]{\norm{\jac \bar\Phi_{\J}(\lambda)- \jac_2 \Phi(\wstoc, \lambda)}^2} + 2b_2. \\
&= 2\underbrace{\norm{\jac_2 \Phi(\wstoc, \lambda)}^2}_{a_1} \underbrace{\Vartilde{\vstocj{\wstoc}}}_{a_2} 
+ 2\underbrace{\Exptilde[\big]{\norm{\vstocj{\wstoc}}^2}}_{a_3} \Vartilde{ \jac_2 \bar\Phi_{\J}(\lambda)}  +2b_2,
\end{align*}
where for the last inequality we used that $\zeta'_i \indep \vstocj{\wstoc}\given \wstoc$ and, in virtue of 
Lemma~\ref{lemA4},  that
\begin{align*}
\Exptilde[\big]{\Delta_v^\top \jac_2 \Phi(\wstoc, \lambda)
    (\jac_2 \bar \Phi_J(\wstoc, \lambda, \zeta)  - \jac_2 \Phi(\wstoc, \lambda))^\top \vstocj{\wstoc}} = 0,
\end{align*}
where $\Delta_v := \vstocj{\wstoc} - \Exptilde{\vstocj{\wstoc}}$.
In the following, we will bound each term of the inequality in order.
\begin{align*}
    a_1 &= \norm{\jac_2 \Phi(\wstoc, \lambda) \mp \jac_2 \Phi(w(\lambda), \lambda)}^2 \\
    &\leq 2\norm{\jac_2 \Phi(w(\lambda), \lambda)}^2 + 2\norm{\jac_2 \Phi(w(\lambda), \lambda) - \jac_2 \Phi(\wstoc, \lambda)}^2 \\
    &\leq 2\LPhi^2 + 2\rhol^2\norm{w(\lambda) - \wstoc}^2.
\end{align*}
Then, applying \Cref{ass:innerbackrates}, and Lemma~\ref{lem:varprop}\ref{lem:varprop_ii}
\begin{align*}
a_2 = \Vartilde{\vstocj{\wstoc}} &\leq \Exptilde{\norm{\vstocj{\wstoc} \mp \voptj{\wstoc} - \vopt{\wstoc}}^2} \\
&\leq  2\hrf(k) + 2 \frac{\monee}{\J (1-\q)^2},
\end{align*}
where in the last inequality, recalling Assumption~\ref{ass:alt}\ref{eq:expgradexp_iv_2}, we used
\begin{equation}\label{eq:added}
\begin{gathered}
    \Exptilde[\big]{\norm{\vopt{\wstoc} - \voptj{\wstoc}}^2} \leq \\
    \norm{(I - \jac_1 \Phi(\wstoc{}, \lambda)^\top)^{-1}}^2 \Exptilde{\norm{ \nabla_1 \fo (\wstoc{}, \lambda) - \nabla_1 \bar\fo_J (\wstoc{}, \lambda)}^2} \leq \\
    \norm{(I - \jac_1 \Phi(\wstoc{}, \lambda)^\top)^{-1}}^2\Vartilde{\nabla_1 \bar \fo_J (\wstoc{}, \lambda)} \leq \\
    \frac{\monee}{\J (1-\q)^2}.
\end{gathered}  
\end{equation}
Furthermore, exploiting \Cref{ass:hypergrad} and \ref{ass:innerbackrates}, and \Cref{lm:normv},
\begin{align*}
    a_3 &= \Exptilde[\big]{\norm{\vstocj{\wstoc} \mp \voptj{\wstoc} \mp \vopt{\wstoc}}^2} \\
    &\leq 2\norm{\vopt{\wstoc}}^2 + 4\Exptilde[\big]{\norm{\vopt{\wstoc} - \voptj{\wstoc}}^2} \\ &\quad+ 4\Exptilde[\big]{\norm{\voptj{\wstoc} - \vstocj{\wstoc}}^2} \\
    &\leq 2\frac{\Bo^2}{(1-\q)^2}  + 4 \frac{\monee}{\J (1-\q)^2} + 4\hrf(k),
\end{align*}
where we used \eqref{eq:added} in the last inequality.
Using the formula for the variance of the sum of independent random variables  and Assumption~\ref{ass:alt} we have
\begin{align*}\label{eq:varphiebound}
    \Vartilde{ \jac \bar\Phi_{\J}(\lambda)} \leq \frac{\mtwo}{\J}, \quad
    \Vartilde{\nabla_2  \bar \fo_{\J}(\wstoc{}, \lambda)} \leq \frac{\mtwoe}{\J}.
\end{align*}
Combining the previous bounds together and defining $\hat \Delta_w : = \norm{w(\lambda) - \wstoc}$ and simplifying some terms knowing that $J > 1$ we get that
\begin{align*}
    \Vartilde{\hat \grad f(\lambda)} \leq &\left(\mtwoe + 4\frac{\mtwo(\Bo^2 + \monee) + \LPhi^2\monee}{(1-\q)^2}\right) \frac{2}{\J} 
+ 8(\LPhi^2+ \mtwo)\hrf(k) 
\\&+ 8 \rhol^2  \Delta^2_w \left(\hrf(k) + \frac{\monee}{J(1-\q)^2} \right).
\end{align*}
The proof is completed by taking the total expectation on both sides of the inequality above.
\end{proof}
\endgroup

\subsection{Proof of Theorem~\ref{th:finalconvdet}}\label{proof:finalconvdet}
\begin{proof} Similarly to the proof of \Cref{th:finalconv}, but with $\J_\s =1$, we obtain a number of samples in $\stot$ iterations which is $N = \sum_{\s = 0}^{\stot -1} 2 (t_\s + 1) = 2 \sum_{\s = 1}^{\stot}   \ceil{c_3 \log (\s)}$ + 1. , if $ \stot > 1$
\begin{align*}
    N &\geq 2 c_3 \sum_{\s = \ceil{\stot/2}}^{\stot} \log (\s) \geq  c_3(\stot/2 -1)\log ( \stot/2), \\
    N &\leq 2c_3 \stot \log\left( \frac{1}{\stot}  \sum_{\s = 1}^{\stot} \s \right) + 4 \stot \leq 4 \stot\left[ c_3\log\left(\frac{\stot+1}{2}\right) + 1  \right].
\end{align*}
Therefore, $N = \Theta (\stot \log(\stot))$.

Since in the deterministic case $\Var{\hat \nabla f (\lambda)} = 0$ and $\Exp{\hat \nabla f(\lambda)} = \hat \nabla f(\lambda)$,  \Cref{th:boundbias}\ref{th:boundbias_ii} and setting $\J = 1$ yields
\begin{equation}\label{eq:finalbound_det}
\begin{aligned}
    & \norm{\hat\nabla f(\lambda_\s)  - \nabla f (\lambda_\s)}^2  \\ &\leq 
     3 \left(\Lol + \frac{\Low\LPhi + \rhol\Bo}{1-\q} + \frac{\rhow\Bo\LPhi}{(1-\q)^2} \right)^2\rf(t_\s)  + 3\LPhi^2\hrf(k_\s) +  3 \rhol^2 \rf(t_\s) \hrf(k_\s).
\end{aligned}
\end{equation}
Now we note that, in view of last result of \Cref{th:firstconv}, we have 
\begin{equation*}
    \rf(t_\s) = \q^{2t_\s} \boundw^2, \qquad 
    \hrf(k_\s) = \q^{2k_\s} \frac{\Bo^2}{(1-\q)^2},
\end{equation*}
and consequently, since $t_\s = k_\s$ and $\q^{2x} \leq q^{x}$ with $ x \geq 1$, we get
\begin{equation*}
     \norm{\hat\nabla f(\lambda_\s) - \nabla f (\lambda_\s)}^2 \leq C \q^{2t_\s},
\end{equation*}
where $C$ incorporates all the constants occurring in \eqref{eq:finalbound_det}.

Recall that $t_\s = \ceil{c_3\log (\s+1)}$ and $c_3 \geq 1/\log(1/\q) > 0$. From the change of base formula we have 
\begin{equation*}
t_\s \geq c_3 \log(1/\q)\log_\q(1/(\s+1)) \geq \log_\q(1/(\s+1)),
\end{equation*} 
since $\log_\q(1/(\s+1)) \geq 0$ due to $\q<1$, $s \geq 0$. Consequently,
\begin{equation*}
    \q^{2 t_\s} \leq \q^{2\log_\q(1/(\s+1))} =  \frac{1}{(\s+1)^{2}}.
\end{equation*}
Hence, we can bound the sum of squared errors as follows.
\begin{equation*}
    \sum_{\s = 0}^{\stot-1}  \norm{\hat\nabla f(\lambda_\s) - \nabla f (\lambda_\s)}^2  \leq \sum_{\s = 0}^{\stot-1} \frac{C}{(\s + 1)^2} \leq \sum_{\s = 1}^{\stot} \frac{C }{\s^2}  \leq \frac{C \pi^2 }{6}.
\end{equation*}
Using this result in combination with \Cref{lm:constrained} we obtain \eqref{eq:biratedet}. Therefore, we have $\frac{1}{\stot}\sum_{\s=0}^{\stot-1}  \Exp{\norm{G_\alpha(\lambda_\s)}^2} \leq \epsilon$ in a number of UL iterations  $S = O(\epsilon^{-1})$. Since we proved that $N = \Theta (\stot \log(\stot))$ we obtain the final sample complexity result.
\end{proof}

\section{Lemmas}
\begin{lemma}\label{lm:lipv}
Let \Cref{ass:hypergrad} be satisfied. Then, for every $w \in \R^d$
\begin{equation}
    \norm{\vopt{w(\lambda)} -\vopt{w}} \leq  \left(\frac{\rhow\Bo}{(1-\q)^2} + \frac{\Low}{1-\q} \right)\norm{w(\lambda) - w}.
\end{equation}
\end{lemma}
\begin{proof}
Let $A_1 := (I- \jac_1 \Phi(w(\lambda), \lambda)^\top)$ and $A_2 = (I- \jac_1 \Phi(w, \lambda)^\top)$.
Then it follows from \Cref{lm:matrixinverse} that
\begin{align*}
    \norm{\vopt{w(\lambda)} -\vopt{w}}
    &\leq \norm{\grad_1 \fo(w(\lambda), \lambda)}\norm{A_1^{-1} - A_2^{-1}} + \Low \norm{A_2^{-1}}  \norm{w(\lambda) - w} \\
    &\leq \norm{\grad_1 \fo(w(\lambda), \lambda)}\norm{A_1^{-1}(A_2 -  A_1)A_2^{-1}} + \frac{ \Low}{1-\q} \norm{w(\lambda) - w} \\
    &\leq \left(\frac{\rhow}{(1-\q)^2}\norm{\grad_1 \fo(w(\lambda),\lambda)} + \frac{\Low}{1-\q} \right)\norm{w(\lambda) - w}.
\end{align*}
Moreover, Assumption~\ref{ass:hypergrad} yields that $\norm{\grad_1 \fo(w(\lambda),\lambda)} \leq \Bo$.
Hence, the statement follows.
\end{proof}

\begin{lemma}\label{lm:normv}
Let \Cref{ass:hypergrad} be satisfied. Then, for every $w \in \R^d$
\begin{equation}
\begin{aligned}
\norm{\vopt{w}} 
\leq \norm{(I - \jac_1\Phi(w, \lambda)^\top)^{-1}} \norm{\grad_1 \fo(w,\lambda)} 
\leq \frac{\Bo}{1 - \q}.
\end{aligned} 
\end{equation}
\end{lemma}
\begin{proof}
It follows from the definition of $\vopt{w}$ and  Assumptions~\ref{ass:hypergrad}\ref{ass:contraction} and \ref{ass:hypergrad}\ref{ass:lipE}
\end{proof}

\section{Standard Lemmas}\label{se:standardlemams}
For completeness, in this section we state without proof some standard results used in the analysis. A proof can be found in \citep{grazzi2021convergence}.
\begin{lemma}
\label{lem:20200601a}
Let $X$ be a random vector with values in $\R^d$
and suppose that $\EE[\norm{X}^2]<+\infty$.
Then $\EE[X]$ exists in $\R^d$ and 
$\norm{\EE[X]}^2 \leq \EE[\norm{X}^2]$.
\end{lemma}

\begin{definition}
Let $X$ be a random vector with value in $\R^d$ 
such that $\EE[\norm{X}^2]<+\infty$. Then
the variance of $X$ is
\begin{equation}\label{eq:variance}
    \Var{X} := \Exp{\norm{X - \Exp{X}}^2}
\end{equation}
\end{definition}

\begin{lemma}[Properties of the variance]
\label{lem:varprop}
Let $X$ and $Y$ be two independent random variables with values in $\R^d$
and let $A$ be a random matrix with values in $\R^{n\times d}$ which is independent on $X$. We also assume that $X,Y$, and $A$ have finite second moment. Then the following hold.
\begin{enumerate}[label={\rm (\roman*)}]
\item\label{lem:varprop_i} 
$\Var{X} = \EE[\norm{X}^2] - \norm{\EE[X]}^2$,
\item\label{lem:varprop_ii} 
$\EE[\norm{X- x}^2] = \Var{X} + \norm{\EE[X] - x}^2$ $ \forall x \in \R^d$. 
Hence, $\Var{X} = \min_{x \in \R^d} \EE[\norm{X- x}^2]$.
\item\label{lem:varprop_iii} 
$ \Var{X + Y} = \Var{X} + \Var{Y}$,
\item\label{lem:varprop_iv} 
$\Var{AX}  \leq \Var{A}\Var{X} + \norm{\Exp{A}}^2\Var{X} + \norm{\Exp{X}}^2\Var{A}$.
\end{enumerate}
\end{lemma}

\begin{definition}\textit{(Conditional Variance)}.
Let $X$ be a random variable with values in $\R^d$ 
and $Y$ be a random variable with values in a measurable space $\mathcal{Y}$. We call \emph{conditional variance} of $X$ given $Y$ the quantity
\begin{equation*}
    \Var{X \given Y} := \Exp{\norm{X -\Exp{X \given Y}}^2 \given Y}.
\end{equation*}{}
\end{definition}{}

\begin{lemma}{(Law of total variance)}\label{lm:totvariance}
Let $X$ and $Y$ be two random variables, we can prove that 
\begin{equation}
    \Var{X} = \Exp{\Var{X \given Y}} + \Var{\Exp{X \given Y}}
\end{equation}
\end{lemma}

\begin{lemma}
\label{lemA4}
Let $\zeta$ and $\eta$ be two independent random variables with values in $\CZ$ and $\mathcal{Y}$ respectively.
Let $\psi\colon \mathcal{Y} \to \R^{m\times n}, \phi\colon \CZ \to \R^{n\times p}$, and
$\varphi\colon \mathcal{Y} \to \R^{p\times q}$ matrix-valued measurable functions. Then
\begin{equation}
\EE[\psi(\eta) (\phi(\zeta) - \EE[\phi(\zeta)] )\varphi(\eta)] = 0
\end{equation}
\end{lemma}

\begin{lemma}
\label{lm:matrixinverse}
Let $A$ be a square matrix such that $\norm{A}\leq q < 1$
Then, $I - A$ is invertible and
\begin{align*}
\norm{(I - A)^{-1}} \leq \frac{1}{1 - q}.
\end{align*}
\end{lemma}

\bibliography{ref}

\end{document}